\documentclass[conference, 11pt, onecolumn]{ieeeconf}

\IEEEoverridecommandlockouts
\usepackage[usenames]{color}
\usepackage{enumerate}
\usepackage{url}
\usepackage{subfig}
\usepackage{amsfonts,mathrsfs}
\usepackage{amssymb,amsmath}
\usepackage{verbatim}
\usepackage{acronym}
\usepackage{mathtools}
\usepackage{cite}
\usepackage{graphicx}
\usepackage{algorithm}
\usepackage[noend]{algpseudocode}

\def\fskip#1{}

\newtheorem{theorem}{Theorem}

\newtheorem{assumption}{Assumption}

\newtheorem{definition}{Definition}
\newtheorem{example}{Example}

\newtheorem{lemma}{Lemma}

\newtheorem{proposition}[theorem]{Proposition}
\newtheorem{remark}{Remark}

\renewcommand{\theenumi}{\arabic{enumi}}

\def\1{{\bf 1}}

\newcommand{\remove}[1]{}

\def\argmin{\mathop{\rm argmin}}
\def\argmax{\mathop{\rm argmax}}

\begin{document}
\title{Learning Stationary Nash Equilibrium Policies in $n$-Player Stochastic Games with Independent Chains}
\author{\authorblockN{S. Rasoul Etesami*}
 \authorblockA{\vspace{-1.5cm}
}
\thanks{*Department of Industrial and Systems Engineering, Department of Electrical and Computer Engineering, and Coordinated Science Lab, University of Illinois Urbana-Champaign, Urbana, IL, USA 61801 (etesami1@illinois.edu). This material is supported by the Air Force Office of Scientific Research under award number FA9550-23-1-0107 and the NSF
CAREER Award under Grant No. EPCN-1944403.}
}

\maketitle
\begin{abstract}
We consider a subclass of $n$-player stochastic games, in which players have their own internal state/action spaces while they are coupled through their payoff functions. It is assumed that players' internal chains are driven by independent transition probabilities. Moreover, players can receive only realizations of their payoffs, not the actual functions, and cannot observe each other's states/actions. For this class of games, we first show that finding a stationary Nash equilibrium (NE) policy without any assumption on the reward functions is interactable. However, for general reward functions, we develop polynomial-time learning algorithms based on dual averaging and dual mirror descent, which converge in terms of the averaged Nikaido-Isoda distance to the set of $\epsilon$-NE policies almost surely or in expectation. In particular, under extra assumptions on the reward functions such as social concavity, we derive polynomial upper bounds on the number of iterates to achieve an $\epsilon$-NE policy with high probability. Finally, we evaluate the effectiveness of the proposed algorithms in learning $\epsilon$-NE policies using numerical experiments for energy management in smart grids.    
\end{abstract}

\section{Introduction}\label{sec:intro}

Since the early work on the existence of a mixed-strategy Nash equilibrium in static noncooperative
games \cite{nash1950equilibrium}, and its extension on the existence of stationary Nash equilibrium policies in dynamic stochastic games \cite{shapley1953stochastic}, substantial research has been done to develop scalable algorithms for computing Nash equilibrium (NE) points in static and dynamic environments. NE provides a stable solution concept for strategic multiagent decision-making systems, which is a desirable property in many applications, such as socioeconomic systems \cite{jackson2010social}, network security \cite{alpcan2010network}, and routing and scheduling \cite{roughgarden2016twenty}, among many others \cite{bacsar2018handbook,cesa2006prediction}.  

In general, computing NE is PPAD-hard \cite{daskalakis2009complexity}, and it is unlikely to admit a polynomial-time algorithm. Thus, to overcome this fundamental barrier, two main approaches have been adapted in the past literature: i) searching for relaxed notions of stable solutions, such as correlated equilibrium \cite{aumann1987correlated}, which includes the set of NE; and ii) searching for NE points in special structured games, such as \emph{potential games} \cite{monderer1996potential} or \emph{concave games} \cite{rosen1965existence}. Thanks to recent advances in the field of learning theory, it is known that some tailored algorithms for finding relaxed notions of equilibrium in case (i) can also be used to compute NE points of structured games in case (ii). For instance, it is known that the so-called \emph{no-regret} algorithms always converge to the set of coarse correlated equilibria \cite{cesa2006prediction}, and they can also be used to compute NE in the class of \emph{socially concave games} \cite{even2009convergence}. However, such results have mainly been developed for static games, in which players repeatedly play the same game and gradually learn the underlying stationary environment. Unfortunately, extension of such results to dynamic stochastic games \cite{shapley1953stochastic,basar1999dynamic}, in which the state of the game evolves as a result of players' past decisions and the realizations of a stochastic nature, imposes major challenges. 

One major challenge that makes the learning task harder as one moves from static games to dynamic stochastic games is the level of uncertainty and nonstationarity introduced on players' decision-making trajectories due to state dynamics. The reason is that players' payoffs depend not only on the actions of others but also on the state of the game, which evolves stochastically as a function of players' actions. Therefore, not only do the players need to learn about each other's actions, but they also need to learn about the state trajectory to respond properly. In this work, we expand over the past literature and extend the existing results for learning stationary NE policies to a subclass of dynamic $n$-player stochastic games with independent chains. Such games provide natural modeling for many applications, such as multiagent wireless communication \cite{altman2007constrained,narayanan2017large,altman2008constrained}, robotic navigation \cite{zhang2021multi}, and energy management in smart grids \cite{etesami2018stochastic}, among others \cite{qiu2021provably,altman2005zero}. In this class of games, there are $n$ players, where each player $i$ has its own finite state space $S_i$ and finite action set $A_i$. Moreover, the state-action transition matrices $P_i(s'_i|a_i,s_i), i=1,\ldots,n$ are assumed to be independent across players. However, the players are coupled through their reward functions such that the reward of player $i$, denoted by $r_i(s,a)$, depends on the states $s\in \prod_{j=1}^nS_j$ and actions $a\in \prod_{j=1}^nA_j$ of all players. 

\subsection{Related Work} 

While learning NE for two-player zero-sum static games has been long known in the literature with a variety of learning algorithms such as fictitious play or regret minimization dynamics, only recently these results have been extended to two-player zero-sum stochastic games \cite{daskalakis2021independent,sayin2022fictitious}. Moreover, learning NE in $n$-player repeated static games under various assumptions on the payoff functions has been extensively studied in the past literature \cite{even2009convergence,mertikopoulos2019learning}. However, these results mainly work for the static setting and cannot be applied directly to dynamic games with stochastic state-action transitions. 

Recent advances in the field of reinforcement learning (RL) \cite{zhang2021multi} have raised substantial interest in developing efficient learning algorithms for computing optimal stationary policies in single-agent Markov decision processes (MDP) \cite{agarwal2021theory,wang2017primal,cardoso2019large}. MDPs are general frameworks that can model many real-world decision-making problems in the face of uncertainties \cite{chen2021primal,zhang2021multi,jin2020efficiently}. Multiagent extensions of MDPs, in which multiple agents (players) want to maximize their payoffs while interacting under a random environment, have been studied using the framework of stochastic games (a.k.a. Markov games) \cite{shapley1953stochastic}. However, the results for equilibrium computation in stochastic games are often much weaker than those for single-agent MDPs, mainly because of the nonstationary environment that is induced by players' decisions \cite{zhang2021multi}. In general, there are strong lower bounds for computing stationary NE in stochastic games that grow exponentially in terms of the number of players \cite{song2021can}. Therefore, prior work has largely focused on the special case of two-player zero-sum stochastic games \cite{zhao2021provably,qiu2021provably,zhang2021multi,tian2021online,sayin2021decentralized,sayin2022fictitious}. For instance, \cite{daskalakis2021independent} provided finite-sample NE convergence result for independent policy gradient methods in two-player zero-sum stochastic games without coordination. More recently, \cite{sayin2022fictitious} adapted the classical fictitious play dynamics with Q-learning for stochastic games and analyzed its convergence properties in two-player zero-sum stochastic games. However, their convergence results are asymptotic and do not provide any explicit convergence rate. 

While two-player zero-sum stochastic games constitute an important basic setting, there are many problems with a large number of players, a situation that hinders the applicability of the existing algorithms for computing a stationary NE. To address this issue, researchers have recently developed learning algorithms for finding NE in special structured stochastic games. For instance, \cite{hambly2021policy} shows that policy gradient methods can find a NE in $n$-player general-sum linear-quadratic games. An application of RL for finding NE in linear-quadratic mean-field games has been studied in \cite{uz2020reinforcement}. Moreover, \cite{zhang2021gradient,leonardos2021global} show that $n$-player Markov potential games, an extension of static potential games to dynamic stochastic games, admit polynomial-time algorithms for computing their NE policies. Unfortunately, the class of Markov potential games is very restrictive because it requires strong assumptions on the existence of a general potential function. In fact, even establishing the existence of such a potential function could be a challenging task \cite{macua2018learning,mguni2021learning}. Motivated by numerous applications of large-scale stochastic games, in this work, we develop polynomial-time learning algorithms for computing NE policies for a natural class of payoff-coupled stochastic games with independent chains \cite{altman2008constrained,qiu2021provably}.

This work is also related to a large body of literature on the \emph{mirror descent} (MD) algorithm \cite{nemirovskij1983problem,beck2003mirror}, and its variant \emph{dual averaging} (DA), also known as \emph{lazy mirror descent} \cite{bubeck2014convex,juditsky2019unifying}. Both MD and DA have been extensively used in online convex optimization \cite{hazan2019introduction,bubeck2014convex}, online learning for MDPs with changing rewards \cite{cardoso2019large}, regret minimization \cite{cesa2006prediction}, and learning of NE in continuous games \cite{gao2020continuous,bravo2018bandit,mertikopoulos2019learning,zhou2017mirror}.  
Although MD and DA algorithms share similarities in their analysis, DA algorithms are believed to be more robust in the presence of noise, while MD algorithms often provide better convergence rates \cite{juditsky2019unifying}. For that reason, we will study both of these variants.

\subsection{Contributions}

We consider the class of payoff-coupled stochastic games with independent chains and develop efficient decentralized learning dynamics toward computing their stationary $\epsilon$-NE policies. Moreover, under certain assumptions on the reward functions, we show that the proposed learning dynamics converge almost surely or in expectation to a stationary $\epsilon$-NE policy with nonasymptotic polynomial-time convergence rates. Our main contributions can be summarized as follows:
\begin{itemize} 
\item Using a reduction of such stochastic games to a static virtual game in terms of occupation measures, we show that for general reward functions, computing a stationary $\epsilon$-NE policy is PPAD-hard, despite the independency of chains.
\item Leveraging the virtual game formulation for general reward functions, we devise novel polynomial-time learning algorithms, which converge in a weaker distance (namely, the averaged Nikaido-Isoda gap) to the set of $\epsilon$-NE policies almost surely. In particular, we show that if the sequence of iterates converges with positive probability, it must be an $\epsilon$-NE policy. 
\item For reward functions that satisfy social concavity property, we show that for any $\alpha\in (0, 1)$, after at most $O(\frac{n}{\alpha \epsilon}\sum_{i=1}^n\frac{|A_i||S_i|}{\delta_i})^2$ episodes, their time-average policies will form an $\epsilon$-NE policy with probability at least $1-\alpha$, where $\delta_i>0$ is a constant and the expected length of each episode is bounded above by $O(\tau\max_i|S_i|)$, where $\tau$ is the worst mixing time of players' internal chains across all stationary policies. Moreover, we improve this bound to $O\big(\sum_{i=1}^n\frac{|A_i|+\log(|S_i||A_i|)}{\epsilon}\big)^2$ for convergence in expectation to an $\epsilon$-NE policy.  
\item We extend our convergence results when the game admits a stable equilibrium and evaluate the effectiveness of the proposed algorithms in learning $\epsilon$-NE policies using numerical experiments. 
\end{itemize}
Our proposed algorithms are simple, easy to implement, and work in a fully decentralized manner. The players take action in the original stochastic game and observe their realized payoffs. They then use this information in the static virtual game, which is a compact representation of the original static game, to update their policies in the space of occupation measures. However, the virtual game can be played only once, and thus the existing learning algorithms for repeated static games such as \cite{even2009convergence} cannot be applied directly to compute a NE. Nevertheless, using a sampling method, we show how to repeatedly play over the original stochastic game and use the collected information in the virtual game to guide the learning dynamics to a NE. In addition to Markov potential games and linear-quadratic stochastic games, this work provides another subclass of $n$-player stochastic games that, under some assumption on players' reward functions, provably admit polynomial-time learning algorithms for finding their stationary $\epsilon$-NE policies.

\subsection{Organization} 

In Section \ref{sec:formulation}, we introduce payoff-coupled stochastic games with independent chains. In Section \ref{sec:dual}, we provide a dual formulation for such games and establish several preliminary results. In Section \ref{sec:alg-design}, we develop an algorithm for learning $\epsilon$-NE policies. In Section \ref{sec:sociall-concave}, we present our convergence results in terms of the averaged Nikaido-Isoda gap function to the set of $\epsilon$-NE policies and without any assumptions on the reward functions and establish polynomial-time convergence rates in high probability or in expectation. In Section \ref{sec-general}, we consider structured reward functions when they are socially concave or allow the existence of a stable equilibrium and obtain stronger polynomial-time convergence rates in terms of the Euclidean distance to the set of $\epsilon$-NE policies. In Section \ref{sec:simulations}, we provide simulations to evaluate the effectiveness of our proposed algorithm. Conclusions are given in Section \ref{sec:conclusion}, and omitted proofs and auxiliary lemmas can be found in Appendix I.     

\bigskip
\section{Problem Formulation}\label{sec:formulation}

We consider a stochastic game with $[n]=\{1,\ldots, n\}$ players, where each player $i\in[n]$ has its own finite set of states $S_i$ and finite set of actions $A_i$. We denote the joint state and action set of players by $S=\prod_{i=1}^n S_i$ and $A=\prod_{i=1}^n A_i$, respectively. At any discrete time $t=0,1,2,\ldots$, we use $s_i^t\in S_i$ and $a_i^t\in A_i$, respectively, to denote the (random) state and action of player $i$. Similarly, we use $s^t=(s_1^t,\ldots,s_n^t)$ and $a^t=(a_1^t,\ldots,a_n^t)$, respectively, to denote the joint state and action vectors for all players at time $t$. Moreover, we let $r_i(s^t,a^t)$ be the (random) reward received by player $i$ at time $t$, where, without loss of generality, we assume that the rewards are normalized such that $r_i\in [0, 1], \forall i\in [n]$. At any time $t$, the information available to player $i$ is given by the history of its \emph{realized} states, actions, and rewards, i.e., $\mathcal{H}^t_i=\{s_i^{\ell},a_i^{\ell},r_i(s^{\ell},a^{\ell}): \ell=0,1,\ldots,{t-1}\}\cup\{s_i^{t}\}$. In particular, we note that player $i$ is not able to observe other players' states and actions, nor can it access the structure of its reward function $r_i(\cdot)$.  At any time $t$, player $i$ takes an action $a_i^t$ based on its information set $\mathcal{H}^t_i$ and receives a reward $r_i(s^t,a^t)$, which also depends on other players' states and actions. After that, the state of player $i$ changes from $s_i^t$ to a new state $s_i^{t+1}$ with probability $P_i(s_i^{t+1}|s_i^{t},a_i^t)$, where the transition probability matrix $P_i$ is known only to player $i$ and is independent of other transition matrices. 

A general policy for player $i$ is a sequence of probability measures $\pi_i=\{\pi^t_i, t=0,1,\ldots\}$ over $A_i$ that at each time $t$ selects an action $a_i\in A_i$ based on past observations $\mathcal{H}^t_i$ with probability $\pi^t_i(a_i|\mathcal{H}_i^t)$. Use of general policies is often computationally expensive, and in practical applications, players are interested in easily implementable policies. In that regard, the class of \emph{stationary} policies constitutes the most well-known class of simple policies, as defined next. 

\begin{definition}
A policy $\pi_i$ for player $i$ is called \emph{stationary} if the probability $\pi_i^t(a_i|\mathcal{H}^t)$ of choosing action $a_i$ at time $t$ depends only on the current state $s_i^t=s_i$, and is independent of the time $t$. In the case of the stationary policy, we use $\pi_i(a_i|s_i)$ to denote this time-independent probability. 
\end{definition}

\begin{assumption}\label{ass:indep}
We assume that the joint transition probability $P(s'|s,a)$ can be factored into independent components $P(s'|s,a)=\prod_{i=1}^n P_i(s'_i|s_i,a_i)$, where $P_i(s'_i|s_i,a_i)$ is the transition model for player $i$. Moreover, we assume that players' policies belong to the class of stationary policies.\footnote{In fact, for single-agent MDPs, restriction to the class of stationary policies is without loss of generality, and the optimal policy can be chosen from stationary policies \cite{altman1999constrained}.}
\end{assumption}

Given some initial state $s^0$, the objective for each player $i\in [n]$ is to choose a stationary policy $\pi_i$ that maximizes its long-term expected average payoff given by
\begin{align}\label{eq:aggregate-reward}
V_i(\pi_i,\pi_{-i})=\mathbb{E}\Big[\lim_{T\to \infty}\frac{1}{T}\sum_{t=0}^{T}r_i(s^t,a^t)\Big],
\end{align} 
where $\pi_{-i}=(\pi_j, j\neq i)$,\footnote{More generally, given a vector $v$, we let $v_{-i}=(v_j, j\neq i)$ be the vector of all coordinates in $v$ other than the $i$th one.} and the expectation is with respect to the randomness introduced by players' internal chains $(P_1,\ldots,P_n)$ and their policies $\pi=(\pi_1,\ldots,\pi_n)$. As is shown in the next section, under the ergodicity Assumption \ref{ass:ergodic}, for any stationary policy profile $\pi$, the limit in \eqref{eq:aggregate-reward} indeed exists and equals \eqref{eq:linear-occupation-form}. This fully characterizes an $n$-player stochastic game with initial state $s^0$, in which each player $i$ wants to choose a stationary policy $\pi_i$ to maximize its expected aggregate payoff $V_i(\pi_i,\pi_{-i})$. In the remainder of the paper, we shall refer to the above payoff-coupled stochastic game with independent chains as $\mathcal{G}=([n], \pi,\{V_i(\pi)\}_{i\in [n]})$.

\begin{definition}
A policy profile $\pi^*=(\pi^*_1,\ldots,\pi^*_n)$ is called a Nash equilibrium (NE) for the game $\mathcal{G}$ if $V_i(\pi^*_i,\pi^*_i)\ge V_i(\pi_i,\pi^*_i)$ for any $i$ and any policy $\pi_i$. It is called an $\epsilon$-NE if $V_i(\pi^*_i,\pi^*_i)\ge V_i(\pi_i,\pi^*_i)-\epsilon$ for any $i$ and any policy $\pi_i$.    
\end{definition}



\bigskip
\section{A Dual Formulation and Preliminaries}\label{sec:dual}

Here, we provide an alternative formulation for the stochastic game $\mathcal{G}$ based on \emph{occupation measures}. Intuitively, from player $j$'s point of view, its long-term expected average payoff depends on the proportion of time that player $j$ spends in state $s_j$ and takes action $a_j$, denoted by its occupancy measure $\rho_j(s_j,a_j)$. Thus, the policy optimization for player $j$ can be viewed as an optimization problem in the space of occupancy measures, where players want to force their chains to spend most of their time in high-reward states. An advantage of optimization in terms of occupancy measures is that due to the independency of players' internal chains, the payoff functions admit a simple decomposable form, which is easier to analyze than the original policy variables. We shall use this dual formulation to develop learning algorithms for finding a stationary $\epsilon$-NE.

To provide the dual formulation, let us use $\pi_j:S_j\to \Delta(A_j)$ to denote players' stationary policies, where $\Delta(A_j)$ is the probability simplex over $A_j$. We can write
\begin{align}\nonumber
\mathbb{E}\big[\sum_{t=0}^{T}r_i(s^t,a^t)\big]&=\sum_{s,a}\sum_{t=0}^{T}\mathbb{P}\big\{s^t=s,a^t=a\big\}r_i(s,a)\cr 
&=\sum_{s,a}\sum_{t=0}^{T}\mathbb{P}\big\{s^t=s\big\}\mathbb{P}\big\{a^t=a|s^t=s\big\}r_i(s,a)\cr 
&=\sum_{s,a}\sum_{t=0}^{T}\mathbb{P}(s^t=s) \prod_{j}\pi_j(a_j|s_j) r_i(s,a)\cr
&=\sum_{s,a}\sum_{t=0}^{T}\Big(\prod_j\mathbb{P}(s_j^t=s_j)\pi_j(a_j|s_j)\Big) r_i(s,a), 
\end{align}
where the third equality results from $\mathbb{P}\big\{a^t=a|s^t=s\big\}=\prod_{j}\pi_j(a_j|s_j)$ because players are using stationary policies that depend only on their own state. Moreover, the last equality holds because, by Assumption \ref{ass:indep}, one can show that $\mathbb{P}(s^t=s)=\prod_j \mathbb{P}(s^t_j=s_j), \forall t\ge 0$ (see Lemma \ref{lemm:indep}). 

On the other hand, following of a stationary policy $\pi_j$ by player $j$ induces a Markov chain over $S_j$ with transition probability 
\begin{align}\label{eq:P-pi}
P^{\pi_j}(s'_j|s_j)=\sum_{a_j\in A_j}P_j(s'_j|a_j,s_j)\pi_j(a_j|s_j).
\end{align}
Thus, assuming that the Markov chain $P^{\pi_j}$ is ergodic, if we let $\nu_j$ be the stationary distribution of $P^{\pi_j}$, we have $\lim_{t\to \infty}\mathbb{P}(s^t_j=s_j)=\nu_j(s_j)$. For any player $j\in [n]$, let us define $\rho_j$ to be the occupancy probability measure that is induced over its state-action set $S_j\times A_j$ by following the stationary policy $\pi_j$, that is 
\begin{align}\label{eq:occupation-def}
\rho_j(s_j,a_j):&=\lim_{T\to \infty}\frac{1}{T}\sum_{t=0}^{T}\mathbb{P}\big\{s_j^t\!=\!s_j,a_j^t\!=\!a_j\big\}\cr 
&=\lim_{T\to \infty}\frac{1}{T}\sum_{t=0}^{T}\mathbb{P}\big\{s_j^t\!=\!s_j\big\}\mathbb{P}\big\{s_j^t\!=\!s_j|a_j^t\!=\!a_j\big\}\cr 
&=\lim_{T\to \infty}\frac{1}{T}\sum_{t=0}^{T}\mathbb{P}\big\{s_j^t\!=\!s_j\big\}\pi_j(a_j|s_j)\cr 
&=\nu_j(s_j)\pi_j(a_j|s_j).
\end{align} 
In other words, $\rho_j(s_j,a_j)$ is the proportion of time that the policy $\pi_j$ spends over state-action $(s_j,a_j)$. Now, we can write the payoff function \eqref{eq:aggregate-reward} as
\begin{align}\nonumber
V_i(\pi_i,\pi_{-i})&=\lim_{T\to \infty}\frac{1}{T}\mathbb{E}\big[\sum_{t=0}^{T}r_i(s^t,a^t)\big]\cr 
&=\sum_{s,a}\Big(\lim_{T\to \infty}\frac{1}{T}\sum_{t=0}^{T}\prod_j\mathbb{P}(s_j^t=s_j)\Big)\prod_{j}\pi_j(a_j|s_j) r_i(s,a)\cr 
&=\sum_{s,a}\Big(\prod_j\nu_j(s_j)\Big)\prod_{j}\pi_j(a_j|s_j) r_i(s,a).
\end{align}
Combining \eqref{eq:occupation-def} with the above relation shows that the payoff \eqref{eq:aggregate-reward} can be written using  occupation measures as
\begin{align}\label{eq:linear-occupation-form}
V_i(\rho_i,\rho_{-i})=\sum_{s,a}\prod_{j=1}^n\rho_j(s_j,a_j)r_i(s,a)=\langle \rho_i, v_i(\rho_{-i}) \rangle,
\end{align} 
where $v_i(\rho_{-i})$ is defined to be a vector of dimension $|S_i||A_i|$ whose $(s_i,a_i)$-th coordinate is given by
\begin{align}\nonumber
v_i(\rho_{-i})_{(s_i,a_i)}=\sum_{s_{-i},a_{-i}}\prod_{j\neq i}\rho_j(s_j,a_j) r_i(s,a).
\end{align}
Moreover, for \emph{fixed} policies of other players with induced occupations $\rho_{-i}$, using \eqref{eq:occupation-def} and \eqref{eq:linear-occupation-form}, the optimal stationary policy for player $i$ is obtained by solving the MDP:
\begin{align}\label{eq:MDP-formulation}
\!\!\max_{\pi_i}\lim_{T\to \infty}\!\frac{1}{T}\!\sum_{t=0}^{T}\!\sum_{s_i,a_i}\!\mathbb{P}\big\{s_i^t\!=\!s_i,a_i^t\!=\!a_i\big\}v_i(\rho_{-i})_{(s_i,a_i)}.
\end{align}  


\begin{lemma}[\cite{altman1999constrained}, Theorem 4.1]\label{lemm:occ-to-policy}
Given any MDP and any feasible occupation measure $\rho_i$ over its set of state-action, one can define a corresponding stationary policy
\begin{align}\label{eq:corresponding-policy}
\pi_i(a_i|s_i)=\frac{\rho_i(s_i,a_i)}{\sum_{a'_i\in A_i}\rho_i(s_i,a'_i)}, \ \ \forall s_i\in S_i, a_i\in A_i,
\end{align}
such that following policy $\pi_i$ in that MDP induces the same occupation measure as $\rho_i$ over the state-action set $S_i\times A_i$, i.e., 
$\lim_{T\to \infty}\!\frac{1}{T}\!\sum_{t=0}^{T}\mathbb{P}\big\{s_i^t\!=\!s_i,a_i^t\!=\!a_i\big\}=\rho_i(s_i,a_i)$.    
\end{lemma}

Therefore, using Lemma \eqref{lemm:occ-to-policy}, the problem of finding the optimal stationary policies for the players reduces to one of finding the optimal occupation measures for them. To characterize the set of feasible occupation measures, from \eqref{eq:occupation-def} we have
$\rho_i(s_i,a_i)=\nu_i(s_i)\pi_i(a_i|s_i)$, such that $\nu_i(s_i)=\sum_{a_i}\rho_i(s_i,a_i), \forall s_i$. Since $\nu_i$ is the stationary distribution of $P^{\pi_i}$, we must have $\sum_{s_i}\nu_i(s_i)P^{\pi_i}(s'_i,s_i)=\nu_i(s'_i), \forall s'_i$, which can be written in terms of occupation variables as $\sum_{s_i,a_i}P_i(s'_i|s_i,a_i)\rho_i(s_i,a_i)=\sum_{a_i}\rho_i(s'_i,a_i), \forall s'_i$. This fully characterizes the set of feasible occupation measures for player $i$ as the feasible points of the following polytope:
\begin{align}\nonumber
\mathcal{P}_i=\Big\{\rho_i\in \mathbb{R}_+^{|S_i\times A_i|}: \sum_{s_i,a_i}\big(P_i(s'_i|s_i,a_i)-\mathbb{I}_{\{s_i=s'_i\}}\big)\rho_i(s_i,a_i)=0\ \forall s'_i,\  \sum_{s_i,a_i}\rho_i(s_i,a_i)=1\Big\},
\end{align}
where $\mathbb{I}_{\{\cdot\}}$ is the indicator function. It is worth noting that since player $i$ knows the transition matrix $P_i$, it can compute its occupation polytope $\mathcal{P}_i$ a priori using at most $O(|S_i||A_i|)$ linear constraints. Thus, the stochastic game $\mathcal{G}$ can be equivalently formulated in a dual (virtual) form as defined next.
\begin{definition}
The \emph{virtual} game associated with the stochastic game $\mathcal{G}$ is an $n$-player continuous-action static game, where the payoff function and the action set for player $i$ are given by $V_i(\rho)=\sum_{s,a}\prod_{j=1}^n\rho_j(s_j,a_j)r_i(s,a)$ and $\mathcal{P}_i$, respectively. 
\end{definition}

\begin{proposition}\label{prop:complexity}
The $n$-player stochastic game $\mathcal{G}$ admits a stationary NE policy. Moreover, finding a stationary NE for $\mathcal{G}$ without any assumption on the reward functions is PPAD-hard.
\end{proposition}
\begin{proof}
To find a stationary NE policy for $\mathcal{G}$, it is sufficient to find a pure-strategy NE $\rho^*$ in the static virtual game, in which case one can use Lemma \ref{lemm:occ-to-policy} to obtain a stationary NE policy $\pi^*$ from $\rho^*$. Thus, it is enough to show that the virtual game admits a pure-NE. This statement is also true by Rosen's theorem \cite{rosen1965existence} for concave games.\footnote{A continuous-action game is called \emph{concave} if for each player $i$ and any fixed actions of the other players, the payoff function of player $i$ is concave and continuous with respect to its own action.} Note that from \eqref{eq:linear-occupation-form}, the payoff function of each player $i$ in the virtual game is concave with respect to its own action $\rho_i$, and the action set $\mathcal{P}_i$ is convex and compact.

Finally, in the special case where each player has only one state, i.e., $|S_i|=1\ \forall i$, the virtual game reduces to a static $n$-player noncooperative matrix game, in which the action set for player $i$ is given by $A_i$, and the entries of the payoff matrix for player $i$ are given by $r_i(a)$. Then, an occupation measure $\rho_i$ for player $i$ can be viewed as a mixed strategy over its action set $A_i$ with the expected payoff function $V_i(\rho_i,\rho_{-i})=\sum_{a}\prod_{j=1}^n\rho_j(a_j)r_i(a)$. Therefore, if we can find a stationary NE for $\mathcal{G}$ in this special case, we can find a mixed-strategy NE in the virtual game. Since the latter is PPAD-hard \cite{daskalakis2013complexity}, finding a stationary NE for $\mathcal{G}$ without any assumption on the reward functions is also PPAD-hard.    
\end{proof}

Using Proposition \ref{prop:complexity}, it is unlikely that scalable learning algorithms can obtain a stationary NE in $\mathcal{G}$ without imposing extra assumptions on the players' reward functions $r_i(s,a)$. That is why in the remainder of this paper, we restrict our attention to the cases where players' reward functions allow the existence of scalable learning algorithms. In fact, using the equivalence between the stochastic game $\mathcal{G}$ and the virtual game, we shall focus only on developing learning algorithms to find an $\epsilon$-NE for the virtual game. This, in view of Lemma \ref{lemm:occ-to-policy} and continuity of the payoff functions, immediately translates to learning algorithms for finding a stationary $\epsilon$-NE for the original game $\mathcal{G}$. However, we should note that for developing a learning algorithm, we cannot solely rely on the virtual game, which is a compact representation of $\mathcal{G}$. In other words, unlike $\mathcal{G}$, which can be played iteratively, the virtual game can be played only once, as it encodes the information of the entire horizon into a single-shot static game. Nevertheless, using a sampling method as in \cite{wang2017primal} that was given for single-agent MDPs, we show how to repeatedly play over $\mathcal{G}$ and use the collected information in the virtual game to guide the learning dynamics to a NE.
   
\section{A Learning Algorithm for $\epsilon$-NE Policies}\label{sec:alg-design}

In this section, we develop our main learning algorithm for the stochastic game $\mathcal{G}$. We first consider the following assumption on the mixing time of players' internal chains, which we shall impose in the remainder of this paper.  
\begin{assumption}\label{ass:ergodic}
For any player $i$ and any stationary policy $\pi_i$ chosen by that player, the induced Markov chain $P^{\pi_i}$ that is given in \eqref{eq:P-pi} is ergodic, and its mixing time is uniformly bounded above by some parameter $\tau$; that is, $\|(v-v')P^{\pi_i}\|_1\leq e^{-\frac{1}{\tau}}\|v-v'\|_1$, for all $i, \pi_i, v,v'\in \Delta(S_i)$.
\end{assumption}
In fact, Assumption \ref{ass:ergodic} is a standard assumption used in the MDP literature  \cite{even2009online,neu2013online,cardoso2019large}, and is much needed to establish meaningful convergence results. Otherwise, if the transition probability matrix $P_i$ of a player $i$  is such that for some policy $\pi_i$ the induced chain $P^{\pi_i}$ takes an arbitrarily large time to mix, then there is no hope that player $i$ can evaluate the performance of policy $\pi_i$ in a reasonably short time.

\begin{definition}
Given a positive constant $\delta_i>0$, we define $\mathcal{P}_i^{\delta_i}=\mathcal{P}_i\cap\{\rho_i\ge \delta_i \boldsymbol{1}\}$ to be the shrunk occupation polytope for player $i$, where $\boldsymbol{1}$ is the column vector of all ones of dimension $|S_i||A_i|$. Similarly, for a vector of positive constants $\delta=(\delta_1,\ldots,\delta_n)$, we define $\mathcal{P}^{\delta}=\prod_{i=1}^{n}\mathcal{P}_i^{\delta_i}$. 
\end{definition}

The shrunk occupation polytope $\mathcal{P}_i^{\delta_i}$ contains all feasible occupation measures for player $i$ with coordinates of at least $\delta_i$. In fact, by restricting player $i$'s occupations to be in $\mathcal{P}_i^{\delta_i}$, we can assure that player $i$ uses stationary policies that choose any action with probability at least $\delta_i$, hence encouraging exploration during the learning process. Thanks to continuity of the payoff functions, working with shrunk polytope $\mathcal{P}^{\delta_i}$ with a sufficiently small threshold $\delta_i$ can only result in a negligible loss in players' payoff functions.

\begin{lemma}\label{lemm:original-to-shrunk}
For any $\epsilon>0$, there exist $\{\delta_i>0, i\in [n]\}$, such that any $\epsilon$-NE for the virtual game with shrunk action sets $\{\mathcal{P}_i^{\delta_i}, i\in [n]\}$ is a $2\epsilon$-NE for the virtual game with action sets $\{\mathcal{P}_i, i\in [n]\}$. In particular, $\delta_i$ can be determined by player $i$ independently of others and based only on its internal transition probability matrix $P_i$.   
\end{lemma}

Now we are ready to describe our main learning algorithm (Algorithm \ref{alg-main}). The algorithm proceeds in different episodes (batches), where each batch contains a random number of time instances. Given any $\epsilon>0$,\footnote{Here, $\epsilon>0$ is an input parameter that can be set freely, and controls the accuracy of the final policies to form a stationary NE policy.} each player first uses Lemma \ref{lemm:original-to-shrunk} to determine a threshold $\delta_i$, and chooses an initial occupation measure $\rho_i^0\in \mathcal{P}^{\delta_i}$. The occupation measure of player $i$ at the beginning of batch $k=0,1,2,\ldots$ is denoted by $\rho_i^k$; during batch $k$, player $i$ chooses actions according to the stationary policy $\pi_i^k$ that is obtained, using \eqref{eq:corresponding-policy}, from the occupation measure $\rho_i^k$. A batch continues for a random number of time instances until each player $i$ has visited all its states $S_i$ at least once. Using Assumptions \ref{ass:ergodic} and \ref{ass:indep}, a simple coupling argument shows that the expected number of time instances in batch $k$ can be upper-bounded by $\tilde{O}(\tau\max_i|S_i|)$. The reason is that the length of batch $k$ equals the maximum cover time among Markov chains $P^{\pi_i^k}, i\in [n]$, whose expected value can be bounded, using Matthews method \cite{levin2017markov}, by the number of states and the mixing time of those chains. Using the collected samples during batch $k$, one can construct an (almost) unbiased estimator $R_i^k$ for the gradient of the payoff function $\nabla_{\rho_i}V_i(\rho)$. The estimator $R_i^k$ is then used in a DA  oracle with an appropriately chosen step-size/regularizer to obtain a new occupation measure $\rho_i^{k+1}$.

\begin{algorithm}[t!]\caption{A Dual Averaging Algorithm for Player $i$}\label{alg-main}

\noindent
{\bf Input:} Initial occupation measure $\rho^0_i\in \mathcal{P}_i^{\delta_i}$, step-size sequence $\{\eta^k_i\}_{k=1}^{\infty}$, a fixed threshold $d$, initial dual score $Y^0=\boldsymbol{0}$, and a strongly convex regularizer $h_i:\mathcal{P}^{\delta_i}_i\to \mathbb{R}$.

\medskip
{\bf For} $k=1,2,\ldots$, do the following:

\begin{itemize} 
\item[$\bullet$] At the end of batch $k-1$, denoted by $\tau^k$, compute 
\begin{align}\nonumber
\pi^k_i(a_i|s_i)=\frac{\rho^k_i(a_i,s_i)}{\sum_{a'_i\in A_i}\rho^k_i(a'_i,s_i)},  \ \ \forall s_i\in S_i, a_i\in A_i,
\end{align}
and keep playing according to this stationary policy $\pi_i^k$ during the next batch $k$. Let $\tau^k_i\ge \tau^k+d$ be the first (random) time such that all states in $S_i$ are visited during $[\tau^k\!+\!d, \tau_i^k]$. Batch $k$ terminates at time $\tau^{k+1}\!=\!\max_i \tau_i^k$. 
\item[$\bullet$] Let $S'_i=S_i$, and $R^k_i\in\mathbb{R}_+^{|S_i||A_i|}$ be a random vector (initially set to zero), which is constructed sequentially during the sampling interval $[\tau^k+d,\tau^{k+1}]$ as follows: 
\begin{itemize}
\item {\bf For} $t=\tau^k+d,\ldots,\tau^{k+1}$ and while $S'_i\neq \emptyset$, player $i$ picks an action $a^t_i$ according to $\pi_i^k(\cdot|s_i^t)$, and observes the payoff $r_i(s^t,a^t)$ and its next state $s_i^{t+1}$. If $s_i^t\in S'_i$, then update $S'_i=S'_i\setminus \{s^t_i\}$, and compute
\begin{align}\label{eq:unbiased-gradient-vector}
R^k_i=R^k_i+\frac{r_i(s^t,a^t)}{\pi^k_i(a_i^t|s_i^t)}\ \boldsymbol{\rm e}_{(s^t_i,a^t_i)},
\end{align} 
where $\boldsymbol{\rm e}_{(s^t_i,a^t_i)}$ is the basis vector with all entries being zero except that the $(s^t_i,a^t_i)$-th entry is 1.\item[]{\bf End For}  
\end{itemize}
\item[$\bullet$] In the end of batch $k$, compute the dual score $Y^{k+1}_i=Y^{k}_i+\eta_i^k R^k_i$, and update the occupation measure:
\begin{align}\label{eq:alg-rho-update}
\rho^{k+1}_i=\argmax_{\rho_i\in \mathcal{P}^{\delta_i}_i}\big\{\langle \rho_i, Y^{k+1}_i\rangle -h_i(\rho_i)\big\}.
\end{align}  
\end{itemize}
{\bf End For}
\end{algorithm}

It is worth noting that the information structure of the game does not allow for centralized computations, as players cannot observe each others' states and actions because of competition or privacy concerns (see Section \ref{sec:simulations} for an example). Moreover, selfish players indeed have incentives to follow Algorithm \ref{alg-main}. The reason is that each player is greedily and independently improving its aggregate payoff by observing its past rewards and using a regularized gradient ascent in the space of occupation measures. 

\begin{remark}
The only point in Algorithm \ref{alg-main} that requires a small amount of coordination among players is the computation of $\tau^{k+1}\!=\!\max_i \tau_i^k$. While this can be performed using a simple signaling mechanism, it can be further relaxed if players take $\tau^{k+1}$ to be a (logarithmic) factor of the maximum expected cover time $\tilde{O}(\tau \max_i|S_i|)$, which, because of independency of chains, introduces a small biased term in the definition of $\epsilon$-NE.
\end{remark}

\section{Convergence Results Using Nikaido-Isoda Gap Function}\label{sec:sociall-concave}

In this section, we analyze the convergence and convergence rate of Algorithm \ref{alg-main} to a stationary $\epsilon$-NE policy measured in terms of the Nikaido-Isoda gap function. The results of this section are general and hold without any assumption on the players' reward functions. In the next section we specialize these results to specific types of reward functions to establish stronger convergence results. To that end, let us consider the following definition which allows us to measure the distance of the iterates of Algorithm \ref{alg-main} from a NE. 

\begin{definition}
The Nikaido-Isoda function $\Psi:\mathcal{P}^{\delta}\times \mathcal{P}^{\delta}\to \mathbb{R}$ is given by
\begin{align}\nonumber
\Psi(\theta,\rho)=\sum_{i=1}^n\big[V_i(\theta_i,\rho_{-i})-V_i(\rho_i,\rho_{-i})\big]. 
\end{align}
\end{definition}

\begin{remark}\label{rem:nikaido}
If $\max_{\theta \in \mathcal{P}^{\delta}} \Psi(\theta,\rho)< \frac{\epsilon}{2}$,\footnote{This maximum value is also known as the total gap function and has been used in the prior literature to measure the distance of an action profile to a NE \cite{golowich2020tight}.} then $\rho\in \mathcal{P}$ must be an $\epsilon$-NE. This is because for any player $i$ and any $\theta_i\in \mathcal{P}^{\delta_i}$ we must have $V_i(\theta_i,\rho_{-i})-V_i(\rho_i,\rho_{-i})< \frac{\epsilon}{2}$, which means that the maximum utility that any player $i$ can obtain by unilaterally deviating from $\rho_i$ to $\theta_i\in \mathcal{P}^{\delta_i}$ can be at most $\frac{\epsilon}{2}$. Thus $\rho$ is an $\frac{\epsilon}{2}$-NE for the virtual game with shrunk action set $\mathcal{P}^{\delta}$ and by Lemma \ref{lemm:original-to-shrunk} it must be an $\epsilon$-NE for the virtual game. 
\end{remark}

To analyze Algorithm \ref{alg-main}, we first show that the estimator $R_i^{k}$ constructed at the end of batch $k$ is nearly an unbiased estimator for the gradient of player $i$'s payoff function. To that end, let us consider the increasing filtration sequence $\{\mathcal{F}^k\}_{k=0}^{\infty}$, which is adapted to the history of random variables $\{\rho^k\}_{k=0}^{\infty}$. More precisely, we let $\mathcal{F}^k$ contain all measurable events that can be realized up to the end of batch $k$ (i.e., until time $\tau^{k+1}$), such that $\rho^{k}$ is $\mathcal{F}^{k-1}$-measurable but $\rho^{k+1}$ is not. We have the following lemma.

\begin{lemma}\label{lemm:almost-unbiased}
Let Assumption \ref{ass:ergodic} hold and assume that each player $i$ follows Algorithm \ref{alg-main}. Conditioned on $\mathcal{F}^{k-1}$, the expected reward vector $R_i^k$ that player $i$ computes at the end of batch $k$ satisfies
\begin{align}\nonumber
\big|\mathbb{E}[R_i^k|\mathcal{F}^{k-1}]-\nabla_{\rho_i} V_i(\rho^k)\big|\leq e^{-\frac{d}{\tau}}\boldsymbol{1},
\end{align} 
where the above inequality is coordinatewise, and the expectation is with respect to the randomness of players' policies and their internal chains.
\end{lemma}

\subsection{Almost Sure Convergence with Dual Averaging Updates}

Using the previous lemma, we are now ready to prove the following main theorem.


\begin{theorem}\label{thm:main-zero-sum}
Assume that the players' regularizer functions are $K$-strongly convex for some $K>0$. Moreover, assume that the sequence of step-sizes $\eta^{\ell}$ satisfy $\sum_{\ell=1}^{\infty}\eta^{\ell}=\infty, \sum_{\ell=1}^{\infty}(\eta^{\ell})^2<\infty$,\footnote{For simplicity, we are assuming that all players have the same sequence of step-sizes $\eta^k=\eta^k_i, \forall i$. Otherwise, all the results remain valid by writing the equations separately for each player $i$. In fact, any sequence $\eta^{\ell}=\frac{1}{\ell^{\beta}}$, where $\beta\in (\frac{1}{2}, 1]$, satisfies the step-size condition of Theorem \ref{thm:main-zero-sum}.} and let $w^k=\sum_{\ell=1}^{k}\eta^{\ell}\ \forall k$. Given $\epsilon>0$, if all players follow Algorithm \ref{alg-main} with $d\ge\tau\ln (\frac{6n}{\epsilon}\sum_{i=1}^n|A_i||S_i|)$, then as $k\to \infty$, almost surely
\begin{align}\nonumber
\max_{\theta \in \mathcal{P}^{\delta}} \sum_{\ell=1}^{k}\frac{\eta^{\ell}}{w^k}\Psi(\theta,\rho^{\ell})< \frac{\epsilon}{2}.
\end{align}
\end{theorem} 
\begin{proof}
If we define $v(\rho)=(\nabla_{\rho_i} V_i(\rho), i\in [n])$, due to the quasilinear structure of the payoff functions $V_{i}(\rho)=\langle \rho_i, v_i(\rho) \rangle$, for any $\ell$ we have
\begin{align}\nonumber
\Psi(\theta,\rho^{\ell})=\sum_{i=1}^n\big[V_i(\theta_i,\rho^{\ell}_{-i})-V_i(\rho^{\ell}_i,\rho^{\ell}_{-i})\big]= \sum_{i=1}^n\langle \nabla_{\rho_i} V_i(\rho^{\ell}), \theta_i-\rho_i^{\ell} \rangle=\langle v(\rho^{\ell}), \theta-\rho^{\ell} \rangle.
\end{align} 
Therefore, we have
\begin{align}\label{eq:Nikido}
\sum_{\ell=1}^{k}\frac{\eta^{\ell}}{w^k}\Psi(\theta,\rho^{\ell})= \sum_{\ell=1}^{k}\frac{\eta_{\ell}}{w^k}\langle v(\rho^{\ell}), \theta-\rho^{\ell} \rangle, \ \ \forall \theta\in \mathcal{P}^{\delta}.
\end{align}
To upper-bound the right-hand side in \eqref{eq:Nikido}, we use the Fenchel coupling $F(p,y)=h(p)+h^*(y)-\langle y, p\rangle$ as a Lyapunov function, where $h(\rho)=\sum_{i=1}^n h_i(\rho_i)$ and $h^*$ is the convex conjugate of $h$. By choosing $\mathcal{X}=\mathcal{P}^{\delta}$ and $h(\rho)=\sum_{i=1}^n h_i(\rho_i)$ in Lemma \ref{lemm:fenchel}, for any $\ell$ and fixed $\theta\in \mathcal{P}^{\delta}$, we can write
\begin{align*}
F(\theta,Y^{\ell+1})&\leq F(\theta,Y^{\ell})\!+\!\langle Y^{\ell+1}\!-\!Y^{\ell}, \Pi(Y^{\ell})-\theta\rangle\!+\!\frac{1}{2K}\|Y^{\ell+1}\!-\!Y^{\ell}\|^2\cr 
&=F(\theta,Y^{\ell})+\eta^{\ell}\langle R^{\ell}, \rho^{\ell}-\theta\rangle\!+\!\frac{(\eta^{\ell})^2}{2K}\|R^{\ell}\|^2.
\end{align*}
By summing this inequality over $\ell=1,\ldots,k$, and rearranging the terms, and because $F(\theta,Y^{k+1})\ge 0$, we obtain 
\begin{align}\label{eq:fenchel-zero-sum}
\sum_{\ell=1}^k\eta^{\ell}\langle R^{\ell}, \theta-\rho^{\ell}\rangle \leq F(\theta,Y^{1}) +\frac{1}{2K}\sum_{\ell=1}^k(\eta^{\ell})^2\|R^{\ell}\|^2.
\end{align}

Next, let us consider martingale difference sequences $\eta^{\ell}\big(R^{\ell}-\mathbb{E}[R^{\ell}|\mathcal{F}^{\ell-1}]\big)$ and $\eta^{\ell}\big(\langle R^{\ell}, \rho^{\ell}\rangle-\mathbb{E}[\langle R^{\ell}, \rho^{\ell}\rangle|\mathcal{F}^{\ell-1}]\big)$, and use $S^{k}=\sum_{\ell=1}^k \eta^{\ell}\big(R^{\ell}-\mathbb{E}[R^{\ell}|\mathcal{F}^{\ell-1}]\big)$ and $Q^{k}=\sum_{\ell=1}^k \eta^{\ell}\big(\langle R^{\ell}, \rho^{\ell}\rangle-\mathbb{E}[\langle R^{\ell}, \rho^{\ell}\rangle|\mathcal{F}^{\ell-1}]\big)$ to denote their corresponding zero-mean martingales, respectively. Then, we have
\begin{align}\nonumber
&\sum_{\ell=1}^{\infty}\big(\frac{\eta^{\ell}}{w^{\ell}}\big)^2\mathbb{E}\big[\big\|R^{\ell}\!-\!\mathbb{E}[R^{\ell}|\mathcal{F}^{\ell-1}]\big\|^2\big|\mathcal{F}^{\ell-1}\big]\leq\sum_{\ell=1}^{\infty}\big(\frac{\eta^{\ell}}{w^{\ell}}\big)^2\mathbb{E}[\|R^{\ell}\|^2|\mathcal{F}^{\ell-1}]\cr 
&\qquad=\sum_{\ell=1}^{\infty}\big(\frac{\eta^{\ell}}{w^{\ell}}\big)^2\sum_{i=1}^{n}\mathbb{E}[\|R_i^{\ell}\|^2|\mathcal{F}^{\ell-1}]\!\leq\! \sum_{\ell=1}^{\infty}\big(\frac{\eta^{\ell}}{w^{\ell}}\big)^2(\sum_{i=1}^{n}\frac{|A_i||S_i|}{\delta_i})\!<\!\infty,
\end{align}
where the last inequality holds because $\sum_{\ell=1}^{\infty}\big(\frac{\eta^{\ell}}{w^{\ell}}\big)^2\leq \sum_{\ell=1}^{\infty}(\eta^{\ell})^2<\infty$. Similarly,
\begin{align}\nonumber
&\sum_{\ell=1}^{\infty}\big(\frac{\eta^{\ell}}{w^{\ell}}\big)^2\mathbb{E}\big[\big\|\langle R^{\ell}, \rho^{\ell}\rangle-\mathbb{E}[\langle R^{\ell}, \rho^{\ell}\rangle|\mathcal{F}^{\ell-1}]\big\|^2\big|\mathcal{F}^{\ell-1}\big]
\leq\sum_{\ell=1}^{\infty}\big(\frac{\eta^{\ell}}{w^{\ell}}\big)^2\mathbb{E}[\|\langle R^{\ell}, \rho^{\ell}\rangle\|^2|\mathcal{F}^{\ell-1}]\cr 
&\qquad\leq\sum_{\ell=1}^{\infty}n\big(\frac{\eta^{\ell}}{w^{\ell}}\big)^2\sum_{i=1}^n\mathbb{E}[\| R^{\ell}_i\|^2\|\rho^{\ell}_i\|^2|\mathcal{F}^{\ell-1}]\leq\sum_{\ell=1}^{\infty}n\big(\frac{\eta^{\ell}}{w^{\ell}}\big)^2\sum_{i=1}^n\mathbb{E}[\| R^{\ell}_i\|^2|\mathcal{F}^{\ell-1}]\cr 
&\qquad\leq \sum_{\ell=1}^{\infty}n\big(\frac{\eta^{\ell}}{w^{\ell}}\big)^2(\sum_{i=1}^{n}\frac{|A_i||S_i|}{\delta_i})<\infty,
\end{align}
where the second inequality uses the Cauchy-Schwartz inequality, and the third inequality holds because $\|\rho_i^{\ell}\|\leq 1$. Thus, using the $L_2$-bounded martingale convergence theorem \cite{hall2014martingale} (Theorem 2.18), almost surely, we have 
\begin{align}\label{eq:martingale-zero-sum}
&\lim_{k\to \infty}\frac{S^k}{w^k}=\lim_{k\to \infty}\sum_{\ell=1}^k\frac{\eta^{\ell}}{w^k}\big(R^{\ell}-\mathbb{E}[R^{\ell}|\mathcal{F}^{\ell-1}]\big)=0, \cr &\lim_{k\to \infty}\frac{Q^k}{w^k}=\lim_{k\to \infty}\sum_{\ell=1}^k\frac{\eta^{\ell}}{w^k}\big(\langle R^{\ell}, \rho^{\ell}\rangle-\mathbb{E}[\langle R^{\ell}, \rho^{\ell}\rangle|\mathcal{F}^{\ell-1}]\big)=0.
\end{align}
Now, using the linearity of expectation and since $\rho^{\ell}$ is $\mathcal{F}^{\ell-1}$-measurable, we can write
\begin{align}\nonumber
&\sum_{\ell=1}^k\frac{\eta^{\ell}}{w^k}\langle \mathbb{E}[R^{\ell}|\mathcal{F}^{\ell-1}], \theta-\rho^{\ell} \rangle\cr 
&=\sum_{\ell=1}^k\frac{\eta^{\ell}}{w^k}\langle R^{\ell}, \theta-\rho^{\ell}\rangle-\sum_{\ell=1}^k\frac{\eta^{\ell}}{w^k}\big(\langle R^{\ell}, \theta-\rho^{\ell} \rangle -\langle \mathbb{E}[R^{\ell}|\mathcal{F}^{\ell-1}], \theta-\rho^{\ell} \rangle\big)\cr 
&=\sum_{\ell=1}^k\frac{\eta^{\ell}}{w^k}\langle R^{\ell}, \theta-\rho^{\ell}\rangle-\sum_{\ell=1}^k\frac{\eta^{\ell}}{w^k}\big\langle R^{\ell}-\mathbb{E}[R^{\ell}|\mathcal{F}^{\ell-1}], \theta \big\rangle+\sum_{\ell=1}^k\frac{\eta^{\ell}}{w^k}\big\langle R^{\ell}-\mathbb{E}[R^{\ell}|\mathcal{F}^{\ell-1}], \rho^{\ell} \big\rangle \cr
&=\sum_{\ell=1}^k\frac{\eta^{\ell}}{w^k}\langle R^{\ell}, \theta-\rho^{\ell}\rangle-\sum_{\ell=1}^k\frac{\eta^{\ell}}{w^k}\big\langle R^{\ell}-\mathbb{E}[R^{\ell}|\mathcal{F}^{\ell-1}], \theta \big\rangle+\sum_{\ell=1}^k\frac{\eta^{\ell}}{w^k} \Big(\langle R^{\ell}, \rho^{\ell}\rangle-\mathbb{E}[\langle R^{\ell}, \rho^{\ell}\rangle|\mathcal{F}^{\ell-1}]\Big)\cr 
&=\sum_{\ell=1}^k\frac{\eta^{\ell}}{w^k}\langle R^{\ell}, \theta-\rho^{\ell}\rangle-\big\langle \frac{S^k}{w^k}, \theta\big\rangle+\frac{Q^k}{w^k}\leq \frac{F(\theta,Y^{1})}{w^k} +\frac{1}{2K}\sum_{\ell=1}^k\frac{(\eta^{\ell})^2}{w_k}\|R^{\ell}\|^2-\big\langle \frac{S^k}{w^k}, \theta\big\rangle+\frac{Q^k}{w^k},
\end{align}
where the last inequality uses \eqref{eq:fenchel-zero-sum}. Moreover, using Lemma \ref{lemm:almost-unbiased} and the Cauchy-Schwartz inequality, we have
\begin{align}\nonumber
\sum_{\ell=1}^k\frac{\eta^{\ell}}{w^k}\langle \mathbb{E}[R^{\ell}|\mathcal{F}^{\ell-1}], \theta-\rho^{\ell} \rangle&\ge \sum_{\ell=1}^k\frac{\eta^{\ell}}{w^k}\langle v(\rho^{\ell}), \theta-\rho^{\ell} \rangle-\sum_{\ell=1}^k\frac{\eta^{\ell}}{w^k}e^{-\frac{d}{\tau}} \|\boldsymbol{1}\|\|\theta-\rho^{\ell}\|\cr 
&\ge \sum_{\ell=1}^k\frac{\eta^{\ell}}{w^k}\langle v(\rho^{\ell}), \theta-\rho^{\ell} \rangle-\frac{\epsilon}{3},
\end{align}  
where the last inequality follows from $\|\theta-\rho^{\ell}\|< 2n$, $\|\boldsymbol{1}\|=\sqrt{\sum_{i}|A_i||S_i|}$, and the choice of $d\ge \tau\ln\big(\frac{6n\sum_i|A_i||S_i|}{\epsilon}\big)$. Thus, using \eqref{eq:Nikido}, for any $\theta\in \mathcal{P}^{\delta}$, we obtain
\begin{align}\label{eq:phi-before-max}
\sum_{\ell=1}^{k}\frac{\eta^{\ell}}{w^k}\Psi(\theta,\rho^{\ell})&= \sum_{\ell=1}^k\frac{\eta^{\ell}}{w^k}\langle v(\rho^{\ell}), \theta-\rho^{\ell} \rangle\leq \sum_{\ell=1}^k\frac{\eta^{\ell}}{w^k}\mathbb{E}[\langle R^{\ell}, \theta-\rho^{\ell} \rangle|\mathcal{F}^{\ell-1}]+\frac{\epsilon}{3} \cr 
&\leq \frac{F(\theta,Y^{1})}{w^k} +\frac{1}{2K}\sum_{\ell=1}^k\frac{(\eta^{\ell})^2}{w_k}\|R^{\ell}\|^2-\big\langle \frac{S^k}{w^k}, \theta\big\rangle+\frac{Q^k}{w^k}+\frac{\epsilon}{3}\cr 
&\leq \frac{F(\theta,Y^{1})}{w^k} +\frac{1}{2K}\sum_{\ell=1}^k\frac{(\eta^{\ell})^2}{w_k}\|R^{\ell}\|^2+\sqrt{n}\| \frac{S^k}{w^k}\|+\frac{Q^k}{w^k}+\frac{\epsilon}{3},
\end{align} 
where the last inequality uses the Cauchy-Schwartz inequality and $\|\theta\|\leq \sqrt{n}, \forall \theta\in \mathcal{P}^{\delta}$.  

If we define a martingale difference $\|\eta^{\ell}R^{\ell}\|^2-\mathbb{E}[\|\eta^{\ell}R^{\ell}\|^2|\mathcal{F}^{\ell-1}]$, and its corresponding zero-mean martingale $T^{k}=\sum_{\ell=1}^{k}(\eta^{\ell})^2\big(\|R^{\ell}\|^2-\mathbb{E}[\|R^{\ell}\|^2|\mathcal{F}^{\ell-1}]\big)$, we have 
\begin{align}\nonumber
\sum_{\ell=1}^{\infty}\frac{(\eta^{\ell})^2}{w^{\ell}}\mathbb{E}\Big[\big\|\|R^{\ell}\|^2-\mathbb{E}[\|R^{\ell}\|^2|\mathcal{F}^{\ell-1}]\big\|\Big|\mathcal{F}^{\ell-1}\Big]&\leq\sum_{\ell=1}^{\infty}\frac{2(\eta^{\ell})^2}{w^{\ell}}\mathbb{E}\big[\|R^{\ell}\|^2\big|\mathcal{F}^{\ell-1}\big]\cr 
&\leq\sum_{\ell=1}^{\infty}\frac{2(\eta^{\ell})^2}{w^{\ell}}(\sum_{i=1}^{n}\frac{|A_i||S_i|}{\delta_i})<\infty,
\end{align}
where the first inequality is obtained by using the triangle inequality, and the last inequality holds by the step-size assumption. Therefore, using the $L_1$-bounded martingale convergence theorem \cite{hall2014martingale} (Theorem 2.18), almost surely we have $\lim_{k\to \infty}\frac{T^{k}}{w^k}=0$. Thus, we can write 
\begin{align}\nonumber
\sum_{\ell=1}^{k}\frac{(\eta^{\ell})^2}{w^k}\|R^{\ell}\|^2=\frac{T^k}{w^k}+\sum_{\ell=1}^{k}\frac{(\eta^{\ell})^2}{w^k}\mathbb{E}[\|R^{\ell}\|^2|\mathcal{F}^{\ell-1}]\leq \frac{T^k}{w^k}+\sum_{\ell=1}^{k}\frac{2(\eta^{\ell})^2}{w^{k}}(\sum_{i=1}^{n}\frac{|A_i||S_i|}{\delta_i}). 
\end{align} 
Substituting the above relation into \eqref{eq:phi-before-max} and taking a maximum from both sides over $\theta\in \mathcal{P}^{\delta}$, we obtain
\begin{align}\label{eq:final-Nikido-bound}
\max_{\theta\in \mathcal{P}^{\delta}}\sum_{\ell=1}^{k}\frac{\eta^{\ell}}{w^k}\Psi(\theta,\rho^{\ell})&\leq \max_{\theta\in \mathcal{P}^{\delta}}\frac{F(\theta,Y_1)}{w^k}\!+\!\sum_{\ell=1}^{k}\frac{(\eta^{\ell})^2}{Kw^{k}}(\sum_{i=1}^{n}\frac{|A_i||S_i|}{\delta_i})\!+\!\frac{T^k}{2Kw^k}\!+\!\sqrt{n}\|\frac{S^k}{w^k}\|\!+\!\frac{Q^k}{w^k}\!+\!\frac{\epsilon}{3}\cr 
&\leq \frac{\sum_{i=1}^n|S_i||A_i|}{w^k}+\frac{\sum_{\ell=1}^k(\eta^{\ell})^2}{Kw^k}(\sum_{i=1}^{n}\frac{|A_i||S_i|}{\delta_i})+\frac{T^k}{2Kw^k}+\sqrt{n}\| \frac{S^k}{w^k}\|+\frac{Q^k}{w^k}+\frac{\epsilon}{3},
\end{align}  
where the last inequality uses $\max_{\theta\in \mathcal{P}^{\delta}}F(\theta,Y_1)=\max_{\theta\in \mathcal{P}^{\delta}} h(\theta)-\min_{\theta\in \mathcal{P}^{\delta}}h(\theta)\leq \sum_{i}|S_i||A_i|$. Using \eqref{eq:martingale-zero-sum} and since $w^k\to \infty$ and $\sum_{\ell=1}^k(\eta^{\ell})^2<\infty$, as $k\to \infty$, almost surely we have $\max_{\theta\in \mathcal{P}^{\delta}}\sum_{\ell=1}^{k}\frac{\eta^{\ell}}{w^k}\Psi(\theta,\rho^{\ell})<\frac{\epsilon}{2}$.
\end{proof}

One immediate corollary of the above theorem is that if with positive probability the sequence of occupation measures generated by Algorithm \ref{alg-main} converges to some limit point $\lim_{\ell\to \infty}\rho^{\ell}=\rho^*$, then the stationary policy $\pi^*$ corresponding to $\rho^*$ is almost surely a stationary $\epsilon$-NE for the game $\mathcal{G}$. The reason is that by continuity of $\Psi(\cdot)$ with respect to its arguments and by conditioning on the event that $\lim_{\ell\to \infty}\rho^{\ell}=\rho^*$, almost surely we have $\max_{\theta\in \mathcal{P}^{\delta}}\Psi(\theta,\rho^*)=\max_{\theta\in \mathcal{P}^{\delta}}\lim_{k\to\infty}\sum_{\ell=1}^{k}\frac{\eta^{\ell}}{w^k}\Psi(\theta,\rho^{\ell})<\frac{\epsilon}{2}.$ This in view of Remark \ref{rem:nikaido} shows that $\rho^*$ must be an $\epsilon$-NE for the virtual game. However, in the following theorem, we provide an alternative proof for this statement that requires weaker conditions on the choice of players' stepsizes and the tuning parameter $d$. The proof of this result is deferred to Appendix I.

\begin{theorem}\label{thm:positive-converge-general}
Given $\epsilon>0$, let $d\ge\tau\ln (\frac{3}{\epsilon}\max_i|A_i||S_i|)$, and suppose that each player $i$ follows Algorithm \ref{alg-main} using a sequence of step-sizes that satisfy $\sum_{k=1}^{\infty}\eta^k_i=\infty$ and $\sum_{k=1}^{\infty}\big(\frac{\eta^k_i}{w^k_i}\big)^2<\infty$, where $w_i^k=\sum_{\ell=1}^k\eta_i^{\ell}$. If with positive probability the sequence of occupation measures generated by Algorithm \ref{alg-main} converges to some point $\lim_{k\to \infty}\rho^k=\rho^*$, then the stationary policy $\pi^*$ corresponding to the limit point $\rho^*$ is a stationary $\epsilon$-NE for the stochastic game $\mathcal{G}$.  
\end{theorem}

In fact, with extra effort, one can leverage the almost-sure convergence result of Theorem \ref{thm:main-zero-sum} to derive high-probability convergence rates for Algorithm \ref{alg-main} in terms of the averaged Nikaido-Isoda gap. This has been shown in the following theorem.

\begin{theorem}\label{thm:convergence-rate-social}
Let $\alpha\in (0, 1)$, and assume that each player follows Algorithm \ref{alg-main} with a $K$-strongly convex regularizer and a sequence of step-sizes $\eta^{\ell}$ that satisfy $\sum_{\ell=1}^{\infty}\eta^{\ell}=\infty$ and $\sum_{\ell=1}^{\infty}(\eta^{\ell})^2<\infty$. Under the same assumptions as in Theorem \ref{thm:main-zero-sum}, $\max\limits_{\theta\in \mathcal{P}^{\delta}}\sum_{\ell=1}^{k}\frac{\eta^{\ell}}{w^k}\Psi(\theta,\rho^{\ell})<\frac{\epsilon}{2}$ with probability at least $1-\alpha$ for \emph{every} $k$ such that
\begin{align}\nonumber
\sum_{\ell=1}^{k}\eta^{\ell}\ge \big(\frac{72n\sum_{\ell=1}^{\infty}(\eta^{\ell})^2}{\alpha\epsilon K}\big)\sum_{i=1}^n\frac{|S_i||A_i|}{\delta_i}.
\end{align}
\end{theorem}
\begin{proof}
Let us consider the last expression \eqref{eq:final-Nikido-bound} in the proof of Theorem \ref{thm:main-zero-sum}, i.e.,  
\begin{align}\label{eq:duplicate-16}
\max_{\theta\in \mathcal{P}^{\delta}}\sum_{\ell=1}^{k}\frac{\eta^{\ell}}{w^k}\Psi(\theta,\rho^{\ell})&\leq \frac{\sum_{i=1}^n|S_i||A_i|}{w^k}+\frac{\sum_{\ell=1}^k(\eta^{\ell})^2}{Kw^k}(\sum_{i=1}^{n}\frac{|A_i||S_i|}{\delta_i})+\frac{T^k}{2Kw^k}+\sqrt{n}\| \frac{S^k}{w^k}\|+\frac{Q^k}{w^k}+\frac{\epsilon}{3}.
\end{align}
To establish high probability convergence rates, we can bound the terms in the above expression as follows. For $k=1,2,\ldots$, define the events $E_k=\{\sup_{\ell\in [k]}\|S^{\ell}\|>\lambda\}$, $F_k=\{\sup_{\ell\in [k]}|Q^\ell|>\lambda\}$, and $G_k=\{\sup_{\ell\in [k]}|T^\ell|>\lambda\}$. Since $\{S^{\ell}\}, \{Q^{\ell}\}$, and $\{T^{\ell}\}$ are martingale sequences, $\{|S^{\ell}\|\}, \{\|Q^{\ell}\|\}$ and $\{\|T^{k}\|\}$ are nonnegative submartingales. Using Doob's maximal inequality for submartingales \cite{hall2014martingale} (Theorem 2.1), we have  
\begin{align}\nonumber
\mathbb{P}(E_k)\leq \frac{\mathbb{E}[\|S^k\|^2]}{\lambda^2}&=\sum_{\ell=1}^{k}\big(\frac{\eta^{\ell}}{\lambda}\big)^2\mathbb{E}\big[\|R^{\ell}\!-\!\mathbb{E}[R^{\ell}|\mathcal{F}^{\ell-1}]\|^2\big]\cr 
&+\sum_{\ell\neq \ell'}\big(\frac{\eta^{\ell}\eta^{\ell'}}{\lambda}\big)\mathbb{E}\big[\langle R^{\ell}\!-\!\mathbb{E}[R^{\ell}|\mathcal{F}^{\ell-1}], R^{\ell'}\!-\!\mathbb{E}[R^{\ell'}|\mathcal{F}^{\ell'-1}]\rangle\big]\cr 
&=\sum_{\ell=1}^{k}\big(\frac{\eta^{\ell}}{\lambda}\big)^2\mathbb{E}\big[\|R^{\ell}\!-\!\mathbb{E}[R^{\ell}|\mathcal{F}^{\ell-1}]\|^2\big]\cr 
&=\sum_{\ell=1}^k\big(\frac{\eta^{\ell}}{\lambda}\big)^2\mathbb{E}\Big[\mathbb{E}\big[\|R^{\ell}\!-\!\mathbb{E}[R^{\ell}|\mathcal{F}^{\ell-1}]\|^2\big|\mathcal{F}^{\ell-1}\big]\Big]\cr 
&\leq \sum_{\ell=1}^{k}\big(\frac{\eta^{\ell}}{\lambda}\big)^2\big(\sum_{i=1}^{n}\frac{|S_i||A_i|}{\delta_i}\big),
\end{align}
where the second equality holds because by conditioning on $\mathcal{F}^{\ell-1}\cup \mathcal{F}^{\ell'-1}$, it is easy to see that the expectation of the cross terms of a martingale different sequence equals zero. Similarly, we can write
\begin{align}\nonumber
\mathbb{P}(G_k)&\leq \frac{\mathbb{E}[\|T^k\|]}{\lambda}\cr
&\leq\sum_{\ell=1}^k\frac{(\eta^{\ell})^2}{\lambda}\mathbb{E}\Big[\mathbb{E}\Big[\big\|\|R^{\ell}\|^2-\mathbb{E}[\|R^{\ell}\|^2|\mathcal{F}^{\ell-1}]\big\|\Big|\mathcal{F}^{\ell-1}\Big]\Big]\cr 
&\leq\sum_{\ell=1}^k\frac{2(\eta^{\ell})^2}{\lambda}\mathbb{E}\Big[\mathbb{E}\big[\|R^{\ell}\|^2|\mathcal{F}^{\ell-1}\big]\Big]\cr
&\leq \sum_{\ell=1}^{k}\frac{2(\eta^{\ell})^2}{\lambda}\big(\sum_{i=1}^{n}\frac{|S_i||A_i|}{\delta_i}\big),
\end{align}
where we note that for bounding $\mathbb{P}(G_k)$, we have used the $L_1$-norm version of Doob's maximal inequality together with the triangle inequality. Finally, to upper-bound $\mathbb{P}(F_k)$, we can write
\begin{align}\nonumber
\mathbb{P}(F_k)&\leq \frac{\mathbb{E}[\|Q^k\|^2]}{\lambda^2}=\sum_{\ell=1}^k\big(\frac{\eta^{\ell}}{\lambda}\big)^2\mathbb{E}\Big[\mathbb{E}\big[\|\langle R^{\ell},\rho^{\ell}\rangle-\mathbb{E}[\langle R^{\ell},\rho^{\ell}\rangle|\mathcal{F}^{\ell-1}]\|^2\big|\mathcal{F}^{\ell-1}\big]\Big]\cr 
&\leq \sum_{\ell=1}^k\big(\frac{\eta^{\ell}}{\lambda}\big)^2\mathbb{E}\Big[\mathbb{E}\big[\|\langle R^{\ell},\rho^{\ell}\rangle\|^2\big|\mathcal{F}^{\ell-1}\big]\Big]\cr 
&\leq \sum_{\ell=1}^k\big(\frac{\eta^{\ell}}{\lambda}\big)^2\mathbb{E}\Big[(\sum_i|A_i||S_i|)\sum_{i,a_i,s_i}\mathbb{E}\big[\big(R^{\ell}_i(s_i,a_i)\rho_i^{\ell}(s_i,a_i)\big)^2|\mathcal{F}^{\ell-1}\big]\Big]\cr 
&=\sum_{\ell=1}^k\big(\frac{\eta^{\ell}}{\lambda}\big)^2(\sum_i|A_i||S_i|)\sum_{i=1}^n\mathbb{E}\Big[\sum_{a_i,s_i}\mathbb{E}\Big[\frac{r^2_i(s_i,a_i;s^{\bar{t}}_{-i},a^{\bar{t}}_{-i})}{\pi_i^{\ell}(a_i|s_i)}(\rho_i^{\ell}(s_i,a_i))^2|\mathcal{F}^{\ell-1}\Big]\Big]\cr 
&\leq \sum_{\ell=1}^k\big(\frac{\eta^{\ell}}{\lambda}\big)^2(\sum_i|A_i||S_i|)\sum_{i=1}^n\mathbb{E}\Big[\sum_{a_i,s_i}\mathbb{E}\Big[(\sum_{a'_i}\rho_i^{\ell}(s_i,a'_i))\rho_i^{\ell}(s_i,a_i)|\mathcal{F}^{\ell-1}\Big]\Big]\cr 
&=\sum_{\ell=1}^k\big(\frac{\eta^{\ell}}{\lambda}\big)^2(\sum_i|A_i||S_i|)\sum_{i=1}^n\mathbb{E}\Big[\sum_{a_i,s_i}(\sum_{a'_i}\rho_i^{\ell}(s_i,a'_i))\rho_i^{\ell}(s_i,a_i)\Big]\cr 
&=\sum_{\ell=1}^k\big(\frac{\eta^{\ell}}{\lambda}\big)^2(\sum_i|A_i||S_i|)\sum_{i=1}^n\mathbb{E}\Big[\sum_{s_i}(\sum_{a_i}\rho_i^{\ell}(s_i,a_i))^2\Big]\cr
&\leq \sum_{\ell=1}^k\big(\frac{\eta^{\ell}}{\lambda}\big)^2(\sum_i|A_i||S_i|)\sum_{i=1}^n\mathbb{E}\Big[\sum_{s_i,a_i}\rho_i^{\ell}(s_i,a_i)\Big]\cr 
&=
\sum_{\ell=1}^{k}n\big(\frac{\eta^{\ell}}{\lambda}\big)^2\big(\sum_{i=1}^{n}|S_i||A_i|\big),
\end{align}
where in the above derivations $\bar{t}$ is the first (random) time at which $s^{\bar{t}}_i=s_i$, and the fourth inequality uses the definition of $\pi^{\ell}_i$ and the fact that $r_i\in [0, 1]$. Moreover, the last inequality holds because $\sum_{a_i} \rho^{\ell}(s_i,a_i)\in (0,1), \forall s_i$. Since $E_{k}\subseteq E_{k+1}\subseteq\ldots$, $F_{k}\subseteq F_{k+1}\subseteq\ldots$, and $G_{k}\subseteq G_{k+1}\subseteq\ldots$, we have 
\begin{align}\nonumber
&\mathbb{P}(\cup_{k=1}^{\infty}E_k)=\lim_{k\to \infty}\mathbb{P}(E_k)\leq \sum_{\ell=1}^{\infty}\big(\frac{\eta^{\ell}}{\lambda}\big)^2\big(\sum_{i=1}^{n}\frac{|S_i||A_i|}{\delta_i}\big),\cr
&\mathbb{P}(\cup_{k=1}^{\infty}F_k)=\lim_{k\to \infty}\mathbb{P}(F_k)\leq \sum_{\ell=1}^{\infty}n\big(\frac{\eta^{\ell}}{\lambda}\big)^2\big(\sum_{i=1}^{n}|S_i||A_i|\big),\cr 
&\mathbb{P}(\cup_{k=1}^{\infty}G_k)=\lim_{k\to \infty}\mathbb{P}(G_k)\leq \sum_{\ell=1}^{\infty}\frac{2(\eta^{\ell})^2}{\lambda}\big(\sum_{i=1}^{n}\frac{|S_i||A_i|}{\delta_i}\big).   
\end{align}  
Now, given any $\alpha\in (0, 1)$, if we take $\lambda= \frac{3 \sqrt{n}}{\alpha}(\sum_{i}\frac{|S_i||A_i|}{\delta_i})\sum_{\ell=1}^{\infty}(\eta^{\ell})^2$, with probability at least $1-\alpha$, none of the events $\cup_{k=1}^{\infty}E_k$, $\cup_{k=1}^{\infty}F_k$ and $\cup_{k=1}^{\infty}G_k$ will occur, i.e., $\|S^k\|\leq \lambda, \|Q^k\|\leq \lambda, \|T^k\|\leq \lambda, \forall k$. Therefore, using \eqref{eq:duplicate-16}, with probability at least $1-\alpha$, for every $k=1,2,\ldots$, we have
\begin{align*}
\max_{\theta\in \mathcal{P}^{\delta}}\sum_{\ell=1}^{k}\frac{\eta^{\ell}}{w^k}\Psi(\theta,\rho^{\ell})&\leq \frac{\sum_{i}|S_i||A_i|}{w^k}+\frac{\sum_{\ell=1}^k(\eta^{\ell})^2}{Kw^k}(\sum_{i=1}^{n}\frac{|A_i||S_i|}{\delta_i})\!+\!\big(\frac{\sqrt{n}\!+\!1\!+\!\frac{1}{2K}}{w^k}\big)\lambda+\frac{\epsilon}{3}\cr
&\leq \frac{12n\sum_{\ell=1}^k(\eta^{\ell})^2}{\alpha Kw^k}(\sum_{i=1}^n\frac{|S_i||A_i|}{\delta_i})+\frac{\epsilon}{3}.
\end{align*} 
Thus, $\max_{\theta\in \mathcal{P}^{\delta}}\sum_{\ell=1}^{k}\frac{\eta^{\ell}}{w^k}\Psi(\theta,\rho^{\ell})\leq \frac{\epsilon}{2}$ with probability $1-\alpha$, for any $k$ such that $w^k\ge \frac{72n}{\alpha\epsilon K}(\sum_i\frac{|S_i||A_i|}{\delta_i})\sum_{\ell=1}^k(\eta^{\ell})^2$.
\end{proof}

\subsection{Expected Convergence Using Mirror Descent Updates}
The results of Theorems \ref{thm:main-zero-sum} and \ref{thm:convergence-rate-social} hold in an almost sure sense with a high-probability convergence rate. In particular, players can use any $K$-strongly convex functions as the regularizer. For instance, if regularizers are taken to be quadratic functions $h_i(\rho_i)=\frac{1}{2}\|\rho_i\|^2$, the policy update step in \eqref{eq:alg-rho-update} reduces to a simple $L_2$-norm projection on the shrunk polytope $\mathcal{P}_i^{\delta_i}$, which can be done quite efficiently in polynomial time. In particular, the use of a DA oracle in the structure of Algorithm \ref{alg-main} can potentially improve the performance of the learning algorithm in the presence of noise due to averaging of the dual scores. However, in this section we show that if players are allowed to choose specific regularizers (e.g., an entropic function), one can obtain improved convergence rates in \emph{expectation} and independent of the parameters $\delta_i$. Motivated by the natural policy gradient that achieves a fast convergence rate for solving MDPs \cite{agarwal2021theory}, in the following theorem, we show that one can obtain a faster expected convergence rate by replacing the policy update rule \eqref{eq:alg-rho-update} in Algorithm \ref{alg-main} with an MD oracle with Kullback-Leibler (KL) divergence $D_{KL}(x,y)=\sum_{r}x_r\log(\frac{x_r}{y_r})$ as the regularizer. Of course, this speedup comes at the cost of a more complex projection in the final step \eqref{eq:alg-rho-update} of Algorithm \ref{alg-main}. 


\begin{theorem}\label{thm:expectation}
Let $d\ge \tau\log (2\sum_{i=1}^n|A_i||S_i|)$, $\sum_{\ell=1}^k\eta^{\ell}=\infty$ and $\sum_{\ell=1}^k(\eta^{\ell})^2<\infty$. Moreover, assume that instead of the projection step \eqref{eq:alg-rho-update} in Algorithm \ref{alg-main}, each player $i$ updates its occupation measure using the following two-step projection:
\begin{align}\label{eq:two-step-project}
&\rho_i^{k+\frac{1}{2}}=\argmax_{\rho_i\in \Delta(S_i\times A_i)}\big\{\langle \eta^kR^k_i, \rho_i \rangle-D_{KL}(\rho_i,\rho_i^{k})\big\},\cr 
&\rho_i^{k+1}=\argmin_{\rho_i\in \mathcal{P}_i} D_{KL}(\rho_i,\rho_i^{k+\frac{1}{2}}).
\end{align}
Then, $\max_{\theta\in \mathcal{P}^{\delta}}\mathbb{E}\big[\sum_{\ell=1}^{k}\frac{\eta^{\ell}}{w^k}\Psi(\theta,\rho^{\ell})\big]\leq \epsilon$ for any $k$ such that 
\begin{align}\nonumber
\sum_{\ell=1}^{k}\eta^{\ell}\ge \frac{1}{\epsilon}\Big(4\sum_{i=1}^n\log(|S_i||A_i|)+(\sum_{i=1}^{n}|A_i|)\sum_{\ell=1}^k(\eta^{\ell})^2\Big).
\end{align}
\end{theorem}
\begin{proof}
First, we note that, without loss of generality, we can normalize the rewards to be in $[-1, 0]$ by simply replacing each $r_i$ with $r_i-1$. Such normalization only shifts players' payoff functions by the same constant $-1$, and all the equilibrium analysis remains as before. In this way, we may assume $R_i^k\leq 0, \forall i, k$. Now, if we use $\Delta=\Delta(S_i\times A_i)$, $X=\mathcal{P}_i$, $x=\rho_i$, and $y=\eta^k R^k_i$ in the statement of Lemma \ref{lemm-KL} (see Appendix I), for any $\theta_i\in \mathcal{P}_i$, we have
\begin{align}\nonumber
D_{KL}(\theta_i,\rho_i^{k+1})-D_{KL}(\theta_i,\rho_i^k)\leq \langle \eta^kR_i^k, \rho_i^k-\theta_i \rangle+\frac{1}{2}\langle (\eta^kR_i^k)^2, \rho_i^k \rangle.  
\end{align}
Taking a conditional expectation from this expression with respect to $\mathcal{F}^{\ell-1}$, we get
\begin{align}\nonumber
\mathbb{E}[D_{KL}(p_i,\rho_i^{k+1})&-D_{KL}(p_i,\rho_i^k)|\mathcal{F}^{\ell-1}]\cr 
&\leq \eta^k\langle \mathbb{E}[R_i^k|\mathcal{F}^{\ell-1}], \rho_i^k-\theta_i \rangle+\frac{(\eta^k)^2}{2}\langle \mathbb{E}[(R_i^k)^2|\mathcal{F}^{\ell-1}], \rho_i^k \rangle\cr 
&\leq \eta^k\langle v_i(\rho), \rho_i^k-\theta_i \rangle+\eta^k|A_i||S_i|e^{-\frac{d}{\tau}}+\frac{(\eta^k)^2}{2}\langle \mathbb{E}[(R_i^k)^2|\mathcal{F}^{\ell-1}], \rho_i^k \rangle\cr 
&=\eta^k\langle v_i(\rho), \rho_i^k-\theta_i \rangle+\eta^k|A_i||S_i|e^{-\frac{d}{\tau}}+\frac{(\eta^k)^2}{2}\sum_{a_i,s_i}\mathbb{E}\Big[\frac{\big(r_i^k(s_i,a_i;s^{\bar{t}}_{-i},a^{\bar{t}}_{-i})\big)^2}{\pi_i^k(s_i,a_i)}\rho^k_i(s_i,a_i)|\mathcal{F}^{\ell-1}\Big]\cr 
&\leq \eta^k\langle v_i(\rho), \rho_i^k-\theta_i \rangle+\eta^k|A_i||S_i|e^{-\frac{d}{\tau}}+\frac{(\eta^k)^2}{2}\sum_{a_i,s_i}\mathbb{E}\big[\frac{\rho^k_i(s_i,a_i)}{\pi_i^k(s_i,a_i)}|\mathcal{F}^{\ell-1}\big]\cr 
&=\eta^k\langle v_i(\rho), \rho_i^k-\theta_i \rangle+\eta^k|A_i||S_i|e^{-\frac{d}{\tau}}+\frac{(\eta^k)^2}{2}\sum_{a_i,s_i}\sum_{a'_i}\rho^k_i(s_i,a'_i)\cr
&=\eta^k\langle v_i(\rho), \rho_i^k-\theta_i \rangle+\eta^k|A_i||S_i|e^{-\frac{d}{\tau}}+\frac{(\eta^k)^2}{2}|A_i|.
\end{align} 
Therefore, for any player $i$, any $\theta_i\in \mathcal{P}_i$, and any time index $\ell=1,2,\ldots$, we have
\begin{align}\nonumber
\eta^{\ell}\langle v_i(\rho^{\ell}), \theta_i-\rho_i^{\ell} \rangle&\leq \mathbb{E}[D_{KL}(\theta_i,\rho_i^{\ell})-D_{KL}(\theta_i,\rho_i^{\ell+1})|\mathcal{F}^{\ell-1}]+\eta^{\ell}|A_i||S_i|e^{-\frac{d}{\tau}}+\frac{(\eta^{\ell})^2}{2}|A_i|.
\end{align}
By taking an unconditional expectation from \eqref{eq:Nikido}, we can write
\begin{align}\nonumber
\mathbb{E}[\sum_{\ell=1}^{k}\frac{\eta^{\ell}}{w^k}\Psi(\theta,\rho^{\ell})]&=\sum_{\ell=1}^k\sum_{i=1}^n\frac{\eta^{\ell}}{w^k}\mathbb{E}[\langle v_i(\rho^{\ell}), \theta_i-\rho_i^{\ell} \rangle]\cr 
&\leq \frac{1}{w^k}\sum_{\ell=1}^k\sum_{i=1}^n\mathbb{E}[\mathbb{E}[D_{KL}(\theta_i,\rho_i^{\ell})-D_{KL}(\theta_i,\rho_i^{\ell+1})|\mathcal{F}^{\ell-1}]]\cr 
&+(\sum_{\ell=1}^{k}\frac{\eta^{\ell}}{w^k})(\sum_{i=1}^n|A_i||S_i|)e^{-\frac{d}{\tau}}+\sum_{\ell=1}^k\frac{(\eta^{\ell})^2}{2w^k}\sum_{i=1}^n|A_i|\cr 
&\leq\frac{1}{w^k}\sum_{\ell=1}^k\sum_{i=1}^n\mathbb{E}[D_{KL}(\theta_i,\rho_i^{\ell})-D_{KL}(\theta_i,\rho_i^{\ell+1})]+\frac{\epsilon}{2}+\sum_{\ell=1}^k\frac{(\eta^{\ell})^2}{2w^k}\sum_{i=1}^n|A_i|\cr 
&=\frac{1}{w^k}\sum_{i=1}^n\mathbb{E}[D_{KL}(\theta_i,\rho_i^{1})-D_{KL}(\theta_i,\rho_i^{k+1})]+\frac{\epsilon}{2}+\sum_{\ell=1}^k\frac{(\eta^{\ell})^2}{2w^k}\sum_{i=1}^n|A_i|\cr 
&\leq\frac{1}{w^k}\sum_{i=1}^n\mathbb{E}[D_{KL}(\theta_i,\rho_i^{1})]+\frac{\epsilon}{2}+\sum_{\ell=1}^k\frac{(\eta^{\ell})^2}{2w^k}\sum_{i=1}^n|A_i|\cr
&\leq\frac{2}{w^k}\sum_{i=1}^n\log(|S_i||A_i|)+\frac{\epsilon}{2}+\sum_{\ell=1}^k\frac{(\eta^{\ell})^2}{2w^k}\sum_{i=1}^n|A_i|.
\end{align}
As the above relation holds for any $\theta\in \mathcal{P}$, we get $\max_{\theta\in \mathcal{P}^{\delta}}\mathbb{E}\big[\sum_{\ell=1}^{k}\frac{\eta^{\ell}}{w^k}\Psi(\theta,\rho^{\ell})\big]\leq \epsilon$ for any $k$ such that $w^k\ge \frac{1}{\epsilon}\big(4\sum_i\log(|S_i||A_i|)+\sum_i|A_i|\sum_{\ell=1}^k(\eta^{\ell})^2\big)$.
\end{proof}


\section{Convergence Results for Games with Structured Reward}\label{sec-general}

As we showed in Proposition \ref{prop:complexity}, finding a stationary NE for the stochastic game $\mathcal{G}$ without any assumption on the reward functions is at least as hard as finding a mixed-strategy NE in normal-form games and is unlikely to admit an efficient learning algorithm. Although the results of the previous section hold generally, the convergence guarantees were in terms of the averaged Nikaido-Isoda gap function. However, one can obtain stronger convergence results by imposing extra assumptions on the reward functions. Therefore, in this section, we consider the stochastic game $\mathcal{G}$ when the virtual game satisfies a certain social concavity or monotonicity property, which allows us to establish stronger convergence results.   

\subsection{Socially Concave Games}

We begin by considering the following socially concave property \cite{even2009convergence}, which has been shown to exist in many static games, such as linear Cournot games, linear resource allocation games, and TCP congestion control games. 

\begin{definition}
A virtual game is called \emph{socially concave} if i) there are positive constants $\lambda_i>0$ such that $\sum_{i=1}^{n}\lambda_iV_i(\rho)$ is a concave function of $\rho$, and ii) for any player $i$ and any fixed $\rho_i$, the payoff function $V_i(\rho_i,\rho_{-i})$ is a convex function of $\rho_{-i}$.  
\end{definition}

Although the social concavity is a condition that is imposed on the virtual game, it can be used to derive conditions on the original game $\mathcal{G}$. Here are two examples.

\begin{example}
The virtual game associated with any two-player zero-sum stochastic game $\mathcal{G}$ is socially concave. That is because by taking $\lambda_1=\lambda_2=1$, due to the zero-sum property of the payoffs, we have $V_1(\rho)+V_2(\rho)=0$, which is constant, and hence a concave function. Moreover, as $V_1(\rho_1,\rho_2)=\sum_{s,a}r_1(s,a)\rho_1(s_1,a_1)\rho_2(s_2,a_2)$, for any fixed $\rho_1$, $V_1(\rho_1,\rho_2)$ a linear (and hence convex) function of $\rho_2$. Similarly, for any fixed $\rho_2$, $V_2(\rho_1,\rho_2)$ is a linear function of $\rho_1$.       
\end{example}

\begin{example}
Consider the original stochastic game $\mathcal{G}$ for which a positive linear combination of reward functions is constant, i.e., $\sum_{i=1}^n\lambda_ir_i(s,a)=c$. Such a situation frequently arises in stochastic resource allocation games in which a constant amount of resources must be shared among the players at each time instance. Using \eqref{eq:linear-occupation-form} we can write $\sum_{i=1}^{n}\lambda_iV_i(\rho)=\sum_{s,a}\prod_{j=1}^n\rho_j(s_j,a_j)\big(\sum_{i=1}^{n}\lambda_ir_i(s,a)\big)=c$, which shows that $\sum_{i=1}^{n}\lambda_iV_i(\rho)$ is a constant (and hence concave) function. Now by change of variables if we let $\rho_i(s_i,a_i)=e^{x_i(s_i,a_i)}$ for some $x_i(s_i,a_i)\in [\ln \delta_i, 0]$, we can express each $V_i(\rho)$ using the new decision variables as $V_i(x)=\sum_{s,a}r_i(s,a)\exp(\sum_{j=1}^n x_j(s_j,a_j))$. Since $V_i(\rho)$ is the sum of multinomials $\prod_{j=1}^n\rho_j(s_j,a_j)$ with nonnegative coefficients $r_i(s_i,a_i)$, it is known that such a change of variable makes the function jointly convex with respect to the new variables $x$ \cite{boyd2004convex}. Therefore, for any fixed $x_i$, $V_i(x_i,x_{-i})$ is a convex function of $x_{-i}$, and the resulting virtual game is socially concave in the space of new decision variables $x$. Of course, one also needs to slightly modify the structure of Algorithm \ref{alg-main} to cope with this change of variables, such as updating the dual scores using $Y_i^{k+1}(s_i,a_i)=Y_i^k(s_i,a_i)+\eta_i^ke^{x_i(s_i,a_i)}R_i^k(s_i,a_i)$ and updating the policy using $\pi^k_i(a_i|s_i)=\exp(x^k_i(a_i,s_i))/\sum_{a'_i\in A_i}\exp(x^k_i(a'_i,s_i))$, which resembles the soft-max policy update rule frequently used in reinforcement learning literature \cite{agarwal2021theory}.        
\end{example}



\begin{theorem}\label{thm:social-concave}
Assume that the virtual game is socially concave, and let $d\ge\tau\ln (\frac{6n}{\epsilon}\sum_{i=1}^n|A_i||S_i|)$ and $\eta^{\ell}=\ell^{-\frac{1}{2}-\beta}$ for some $\beta>0$. Given any $\epsilon>0$, we have:
\begin{itemize}
\item[(a)] If players follow Algorithm \ref{alg-main} with $K$-strongly convex regularizers, then almost surely $\bar{\rho}^k=\sum_{\ell=1}^{k}\frac{\eta^{\ell}}{w^k}\rho^{\ell}$ is an $\epsilon$-NE as $k\to \infty$. In particular, given $\alpha\in (0, 1)$, with probability at least $1-\alpha$, $\bar{\rho}^k$ is an $\epsilon$-NE for any $k\ge O\big(\sum_{i}\frac{n|S_i||A_i|}{\alpha\epsilon \delta_i}\big)^\frac{2}{1-2\beta}$. 
\item[(b)] If players follow Algorithm \ref{alg-main} with KL regularizers and the two-stage update rule \eqref{eq:two-step-project}, then $\mathbb{E}[\bar{\rho}^k]$ is an $\epsilon$-NE for any $k\ge O\big(\sum_{i}\frac{\log(|S_i||A_i|)+|A_i|}{\epsilon}\big)^\frac{2}{1-2\beta}$.
\end{itemize}
\end{theorem}
\begin{proof}
(a) Let us consider an arbitrary $\theta\in \mathcal{P}^{\delta}$ and fix it. We can write\footnote{Here, we have used the scaled version of the payoffs $\lambda_iV_i$ instead of $V_i$ in the definition of $\Psi(\cdot)$.}
\begin{align}\label{eq:social-concave-convex}
\Psi(\theta,\bar{\rho}^k)&=\sum_{i=1}^n\lambda_i V_i(\theta_i,\bar{\rho}^k_{-i})-\sum_{i=1}^n\lambda_iV_i(\bar{\rho}^k)\cr 
&\leq \sum_{i=1}^n\lambda_i V_i(\theta_i,\bar{\rho}^k_{-i})-\sum_{\ell=1}^{k}\frac{\eta_{\ell}}{w^k}\big(\sum_{i=1}^n\lambda_iV_i(\rho^{\ell})\big)\cr 
&\leq \sum_{i=1}^n\lambda_i \big(\sum_{\ell=1}^{k}\frac{\eta_{\ell}}{w^k}V_i(\theta_i,\rho^{\ell}_{-i})\big)-\sum_{\ell=1}^{k}\frac{\eta_{\ell}}{w^k}\big(\sum_{i=1}^n\lambda_i V_i(\rho^{\ell})\big)\cr
&=\sum_{\ell=1}^{k}\frac{\eta_{\ell}}{w^k}\sum_{i=1}^n\lambda_i\big[V_i(\theta_i,\rho^{\ell}_{-i})-V_i(\rho^{\ell})\big]=\sum_{\ell=1}^{k}\frac{\eta_{\ell}}{w^k}\Psi(\theta,\rho^{\ell}),
\end{align} 
where the first inequality is valid because $\sum_{i}\lambda_iV_i(\rho)$ is a concave function of $\rho$, and the second inequality holds because $V_i(\theta_i,\rho_{-i})$ is a convex function of $\rho_{-i}$. By talking maximum from both sides of the above inequality with respect to $\theta$ and using Theorem \ref{thm:main-zero-sum}, we obtain
\begin{align}\nonumber
\max_{\theta\in \mathcal{P}^{\delta}}\Psi(\theta,\bar{\rho}^k)\leq \max_{\theta\in \mathcal{P}^{\delta}}\sum_{\ell=1}^{k}\frac{\eta_{\ell}}{w^k}\Psi(\theta,\rho^{\ell})<\frac{\epsilon}{2}.
\end{align}
Thus, as $k\to \infty$, almost surely we have $\max_{\theta \in \mathcal{P}^{\delta}} \Psi(\theta,\bar{\rho}^k)<\frac{\epsilon}{2}$, which in view of Remark \ref{rem:nikaido} shows that $\bar{\rho}^k$ forms an $\epsilon$-NE. Moreover, due to the choice of stepsize $\eta^{\ell}=\ell^{-\frac{1}{2}-\beta}$, we have $\sum_{\ell=1}^{k}\eta^{\ell}=\Theta(k^{\frac{1}{2}-\beta})$ and $\sum_{\ell=1}^{k}(\eta^{\ell})^2=O(1)$. Thus, using Theorem \ref{thm:convergence-rate-social}, with probability at least $1-\alpha$, $\bar{\rho}^k$ is an $\epsilon$-NE for any $k\ge O\big(\frac{n}{\alpha\epsilon}\sum_{i}\frac{|S_i||A_i|}{\delta_i}\big)^\frac{2}{1-2\beta}$. As $\beta\to 0$, this gives an asymptotic convergence rate that scales only quadratically in the number of players and the size of their state/action spaces.

(b) As is shown in \eqref{eq:social-concave-convex}, because of the social concavity assumption, the Nikaido-Isoda function $\Psi(\cdot)$ is a convex function with respect to its second argument. Thus, using Jensen's inequality, we have
\begin{align}\nonumber
\max_{\theta\in \mathcal{P}^{\delta}}\Psi(\theta,\mathbb{E}[\bar{\rho}^k])\leq \max_{\theta\in \mathcal{P}^{\delta}}\mathbb{E}[\Psi(\theta,\bar{\rho}^k)]\leq \max_{\theta\in \mathcal{P}^{\delta}}\mathbb{E}[\sum_{\ell=1}^{k}\frac{\eta_{\ell}}{w^k}\Psi(\theta,\rho^{\ell})]\leq \frac{\epsilon}{2}, 
\end{align}
where the last inequality follows from Theorem \ref{thm:expectation} for any time instance $k$ such that $w^k\ge \frac{1}{\epsilon}\big(4\sum_i\log(|S_i||A_i|)+\sum_i|A_i|\sum_{\ell=1}^k(\eta^{\ell})^2\big)$. Substituting $w^k=\Theta(k^{\frac{1}{2}-\beta})$ and $\sum_{\ell=1}^{k}(\eta^{\ell})^2=O(1)$ into this relation shows that $\mathbb{E}[\bar{\rho}^k]$ is an $\epsilon$-NE for any $k$ such that $k\ge O\big(\sum_{i}\frac{\log(|S_i||A_i|)+|A_i|}{\epsilon}\big)^\frac{2}{1-2\beta}$. As $\beta\to 0$, this gives an asymptotic convergence rate that scales only logarithmically in terms of the size of players' state spaces.
\end{proof}

For $n=2$ players, Theorem \ref{thm:social-concave} provides improved convergence rates compared to those given in \cite{qiu2021provably} (Theorems 3.1 and 3.2). However, this speedup comes at a cost as the convergence result of \cite{qiu2021provably} holds under weaker assumptions on the ergodicity and mixing time of the players' Markov chains. In fact, unlike the learning algorithm in \cite{qiu2021provably} that uses a combination of UCB and fictitious play, Algorithm \ref{alg-main} is arguably simpler to implement and works for any number of players assuming social concavity.

\subsection{Games with Stable Equilibrium}

In this section, we consider another special case of reward functions and show that if the virtual game admits a \emph{stable} NE as defined next, almost surely a subsequence of occupation measures generated by Algorithm \ref{alg-main} will converge to an $\epsilon$-NE. 

\begin{definition}\label{def:stable}
An occupation profile $\rho^*$ is called a \emph{stable} NE for the virtual game if $\langle v(\rho), \rho^*-\rho\rangle\ge 0, \forall \rho\in \mathcal{P}$, with equality if and only if $\rho=\rho^*$, where $v(\rho)=(v_i(\rho_{-i}),i\in [n])$ is the vector of players' payoff gradients with respect to their own strategies. If, in addition, $\langle v(\rho), \rho^*-\rho\rangle\ge L\|\rho^*-\rho\|^2, \forall \rho\in \mathcal{P}$, for some constant $L>0$, then $\rho^*$ is called an $L$-strongly stable NE. 
\end{definition}

Since the virtual game with payoff functions $V_i(\rho)=\langle \rho_i, v_i(\rho_{-i})\rangle$ is a concave game, if we use NE characterization for concave games \cite{rosen1965existence,mertikopoulos2019learning}, an occupation profile $\rho^*\in \mathcal{P}$ is a NE if and only if $\langle v(\rho^*), \rho^*-\rho  \rangle\ge 0, \forall \rho\in \mathcal{P}$. This, in view of Definition \ref{def:stable}, shows that a stable NE can be viewed as a NE that is globally attractive. To measure the convergence speed of Algorithm \ref{alg-main} to a stable NE $\rho^*$, we again consider the averaged Nikaido-Isoda function with respect to an stable equilibrium $\rho^*$, i.e.,
\begin{align*}
d(k):=\sum_{\ell=1}^{k}\frac{\eta^{\ell}}{w^k}\Psi(\rho^*,\rho^{\ell}),
\end{align*} 
where $\{\rho^{\ell}\}_{\ell=1}^{\infty}$ is the sequence of iterates generated by Algorithm \ref{alg-main}.\footnote{In fact, if $\lim_{k\to \infty}d(k)=0$, almost surely there must be a subsequence $\{\rho^{\ell_j}\}_{j=1}^{\infty}$ such that $\lim_{j\to \infty}\rho^{\ell_j}=\rho^*$, which justifies the choice of $d(k)$ for measuring the distance of the iterates to $\rho^*$.} The following theorem shows that the in the presence of stable NE, the accumulation points of the sequence $\{\rho^{\ell}\}_{\ell=1}^{\infty}$ can be considered as good estimates for the $\epsilon$-NE policies.

\begin{theorem}\label{thm:final-stable}
Assume that the virtual game admits a stable NE $\rho^*$. If each player follows Algorithm \ref{alg-main} with $d\ge \tau \ln\big(\frac{6n}{\epsilon}\sum_i|A_i||S_i|\big)$, a $K$-strongly convex regularizer, and a sequence of step-sizes $\eta^{\ell}$ that satisfy $\sum_{\ell=1}^{\infty}\eta^{\ell}=\infty$ and $\sum_{\ell=1}^{\infty}(\eta^{\ell})^2<\infty$, then almost surely $\limsup_k d(k)<\epsilon$. If, in addition, $\rho^*$ is $L$-strongly stable, almost surely there are infinitely many $k$ such that $\|\rho^*-\rho^{k}\|^2 <\frac{\epsilon}{L}$. Moreover, with probability at least $1-\alpha$, we have $d(k)<\epsilon$ for any $k$ such that $w^k\ge \frac{72n}{\alpha\epsilon K}(\sum_i\frac{|S_i||A_i|}{\delta_i})\sum_{\ell=1}^{\infty}(\eta^{\ell})^2$. 
\end{theorem} 
\begin{proof}
By choosing $\theta=\rho^*$ in relation \eqref{eq:final-Nikido-bound}, we have
\begin{align}\label{eq:stable-final-martingale}
d(k)=\sum_{\ell=1}^{k}\frac{\eta^{\ell}}{w^k}\Psi(\rho^*,\rho^{\ell})&\leq \frac{\sum_{i=1}^n|S_i||A_i|}{w^k}+\frac{\sum_{\ell=1}^k(\eta^{\ell})^2}{Kw^k}(\sum_{i=1}^{n}\frac{|A_i||S_i|}{\delta_i})+\frac{T^k}{2Kw^k}+\sqrt{n}\| \frac{S^k}{w^k}\|+\frac{Q^k}{w^k}+\frac{\epsilon}{3},
\end{align}
where we note that by the martingale convergence theorem and the choice of step-size, as $k\to \infty$, all the terms on the right-hand side of the above expression (except the term $\frac{\epsilon}{3}$) almost surely converge to zero. Therefore, almost surely we have $\limsup_{k} d(k)< \epsilon$. Moreover, using the multilinear structure of the payoff functions $V_i(\rho)=\langle \rho_i, v_i(\rho_{-i})\rangle$, we have $d(k)=\sum_{\ell=1}^{k}\frac{\eta^{\ell}}{w^k}\Psi(\rho^*,\rho^{\ell})=\sum_{\ell=1}^{k}\frac{\eta^{\ell}}{w^k}\langle v(\rho^{\ell}), \rho^*-\rho^{\ell}\rangle$. Since $\rho^*$ is a stable NE, each of the summands $\langle \rho^*-\rho^{\ell},  v(\rho^{\ell}) \rangle, \ell=1,2,\ldots$, is nonnegative, which together with $\limsup_{k} d(k)< \epsilon$ implies that almost surely there exists a subsequence $\{\rho^{\ell_j}\}_{j=1}^{\infty}$, such that $0\leq \langle v(\rho^{\ell_j}), \rho^*-\rho^{\ell_j}\rangle< \epsilon, \forall j$. In particular, if $\rho^*$ is strongly stable, almost surely $\|\rho^*-\rho^{\ell_j}\|<\frac{\epsilon}{L}, \forall j$, which completes the first part.

To prove the high probability convergence rate, we can again define the events $E_k, F_k$, and $G_k$ as before and conclude that for any $\alpha\in (0,1)$, if we take $\lambda=\frac{3\sqrt{n}}{\alpha}(\sum_i\frac{|S_i||A_i|}{\delta_i})\sum_{\ell=1}^{\infty}(\eta^{\ell})^2$, with probability at least $1-\alpha$, we have $\|S^k\|\leq \lambda$, $\|Q^k\|\leq \lambda$, and $\|T^k\|\leq \lambda$. Substituting these relations into \eqref{eq:stable-final-martingale}, we conclude that with probability at least $1-\alpha$,
\begin{align*}
d(k)&\leq \frac{\sum_{i=1}^n|S_i||A_i|}{w^k}+\frac{\sum_{\ell=1}^k(\eta^{\ell})^2}{Kw^k}(\sum_{i=1}^{n}\frac{|A_i||S_i|}{\delta_i})+\big(\frac{\sqrt{n}+1+\frac{1}{2K}}{w^k}\big)\lambda+\frac{\epsilon}{3}\cr 
&\leq \frac{12n\sum_{\ell=1}^k(\eta^{\ell})^2}{\alpha Kw^k}(\sum_{i=1}^n\frac{|S_i||A_i|}{\delta_i})+\frac{\epsilon}{3}.
\end{align*}
Thus, for every $k$ such that $w^k\ge \frac{72n}{\alpha\epsilon K}(\sum_i\frac{|S_i||A_i|}{\delta_i})\sum_{\ell=1}^{\infty}(\eta^{\ell})^2$, with probability at least $1-\alpha$ we have $d(k)\leq \epsilon$.  
\end{proof}

\section{Numerical Results}\label{sec:simulations}

In this section, we provide several numerical results to illustrate the effectiveness of the proposed Algorithm \ref{alg-main} in learning $\epsilon$-NE policies, even in the absence of social concavity or the existence of a stable equilibrium. Motivated by applications such as energy management in smart grids, we first consider the following stochastic game model.     

\subsection{Game Model} 
Consider an energy market with one utility company and $n$ players, which can both produce and consume energy. Each player generates energy using its solar panel or wind turbine and is equipped with a storage device that can store the remaining energy at the end of each day $t\in \mathbb{Z}_+$. Let $s^t_i$ denote the (quantized) amount of stored energy of player $i$ at the beginning of day $t$ with maximum storage capacity $C$. Moreover, let $g_i^t$ be a random variable denoting the amount of harvested energy for player $i$ at the end of day $t$, whose distribution is determined by the weather conditions on that day. Now if we denote the total amount of energy consumed by player $i$ during day $t$ by $a_i^t$, then the stored energy at the end of day $t$ (or the beginning of day $t+1$) is given by $s_i^{t+1}=\min\{C, g_i^t+(s_i^t-a_i^t)^+\}$, where $(s_i^t-a_i^t)^+=\max\{0, s_i^t-a_i^t\}$. In particular, player $i$ needs to purchase $(a_i^t-s_i^t)^+$ units of energy from the utility company on day $t$ to satisfy its demand on that day. On the other hand, the utility company sets the energy price as a function of total demands $\{(a_i^t-s_i^t)^+, i\in [n]\}$, which is given by $p(a^t,s^t)$. If $u_{i}(a^t_i)$ denotes the utility that player $i$ derives by consuming $a_i^t$ units of energy, then the reward of player $i$ at time $t$ is given by $r_i(a^t,s^t)=u_{i}(a^t_i)-p(a^t,s^t)\cdot(a_i^t-s_i^t)^+$. In particular, if players are at distant locations, they likely experience independent weather conditions, so their transition probability models that are governed by stochasticity of $\{g^t_i, i\in [n]\}$ will be independent. In this game, players want to adopt consumption policies to maximize their aggregate rewards despite not being able to observe others' states/actions.

\subsection{Parameters Setup} 
In order to simulate the performance of Algorithm \ref{alg-main} for the above game setup, we consider the following choice of parameters.  
\begin{itemize}
\item Each player $i$ has a storage capacity size of $C=7$, and $S_i=A_i=\{0,\ldots,7\}$.
\item Each player $i$ has its own i.i.d. random harvested energy that is uniformly distributed over $g^t_i\!\sim\! \mbox{Unif}\{0,\ldots,G_i-1\!\}$, where $G_i$ is a player-specific constant.
\item The price function is given by the players' aggregate demand, i.e., $p(a^t,s^t)=\lambda \sum_{i=1}^n (a_i-s_i)^+$, where $\lambda$ is a parameter set by the utility company. We will consider two cases of $\lambda=0$ (i.e., free energy) and $\lambda=1.5$.      
\item We set $u_i(a_i)=a_i^2\ \forall i$, and the reward functions $r_i(a,s)$ are normalized by $C^2$ (which is the maximum reward a player can obtain) to ensure $r_i\in [0, 1]$.  
\end{itemize}

\begin{figure}[t!]
\vspace{-0.5cm}
    \centering
    \hspace{-0.3cm}
    \subfloat[\centering ]{{\includegraphics[width=7cm]{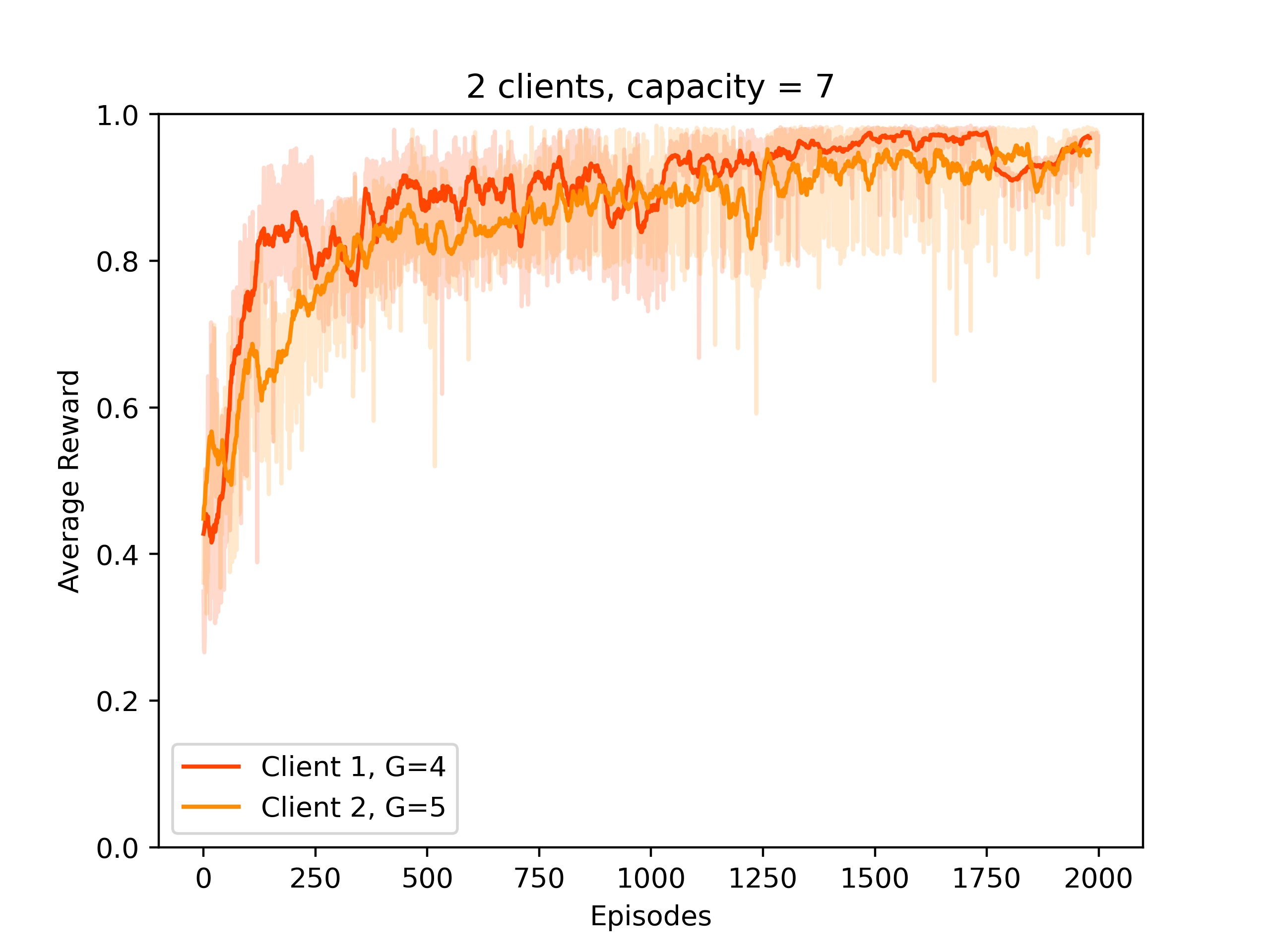} }}%
    \qquad
    \subfloat[\centering ]{{\includegraphics[width=7cm]{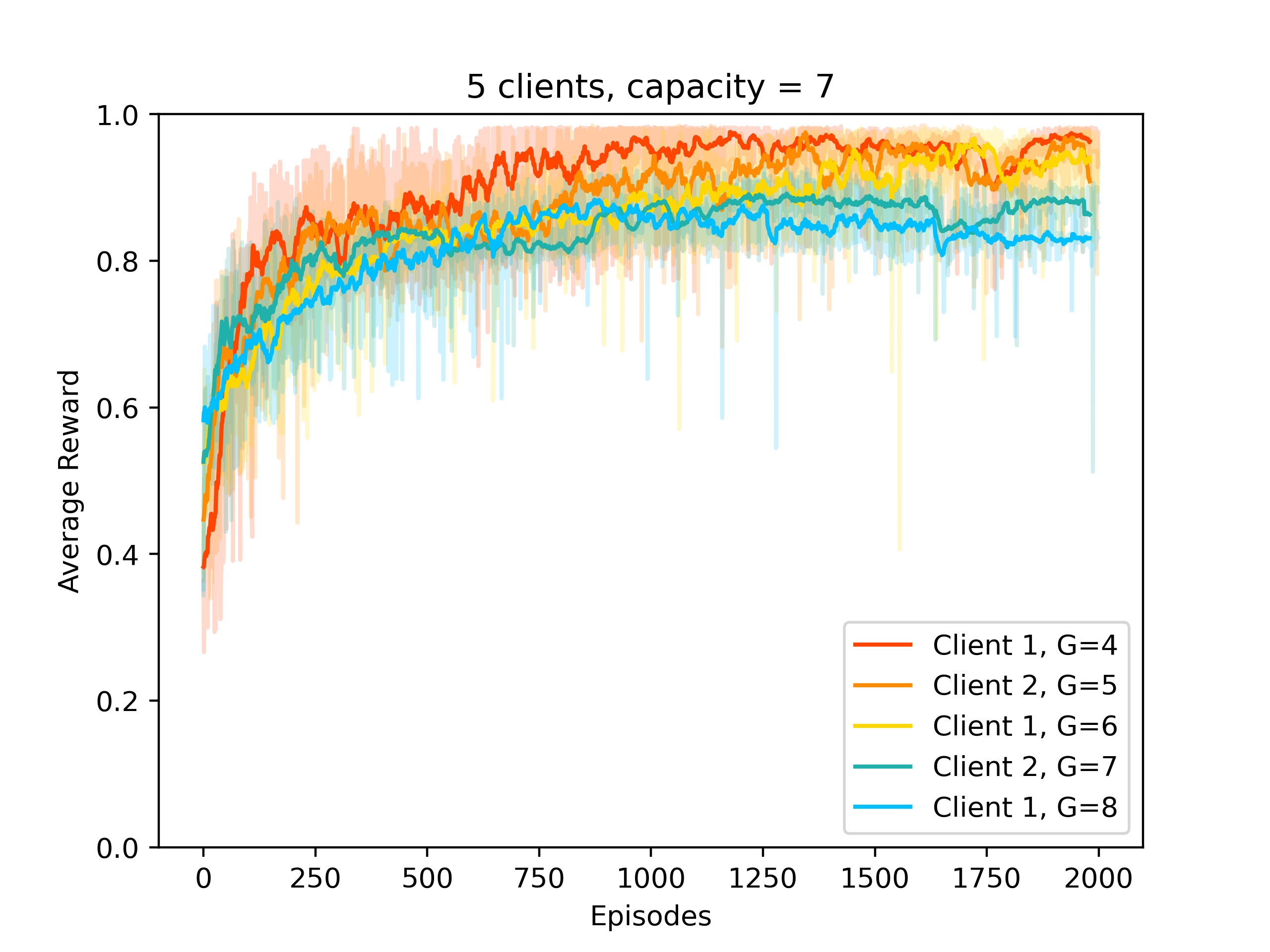} }}%
    \caption{Players' average reward trajectories versus the number of episodes ($\lambda=0$).}%
    \label{fig:lambda0}
\end{figure}

We run Algorithm \ref{alg-main} with the choice of $d = 500$, regularizers $h_i(x_i) = 1000 x_i^2\ \forall i$, shrunk parameter $\delta_i = \frac{1}{20|A_i||S_i|} = 7.8\times 10^{-4}\ \forall i$, and $\eta^{\ell} = \frac{0.02}{\ell}$. For the price parameter $\lambda=0$, i.e., when the energy is free, the players aggressively consume energy because they can always get access to free energy. Thus, it is not surprising that the equilibrium payoffs will approach 1 as depicted in Figure \ref{fig:lambda0} for both cases of $n=2$ players and $n=5$ players.

For the price parameter $\lambda=1.5$, the players' averaged reward trajectories versus the number of episodes during the run of the algorithm are illustrated in Figure \ref{fig:lambda1.5}-(a) when there are $n=2$ players, and in Figure \ref{fig:lambda1.5}-(b) when there are $n=5$ players. As can be seen, the trajectories in both cases converge fast to equilibrium trajectories with less oscillation as the number of episodes increases. It can be seen from the figures that in both cases, the equilibrium averaged rewards are obtained nearly after $k\leq 2000$ episodes, which scales polynomially in terms of the game parameters and the numbers of players.

\begin{figure}
    \centering
    \hspace{-0.3cm}
    \subfloat[\centering ]{{\includegraphics[width=7cm]{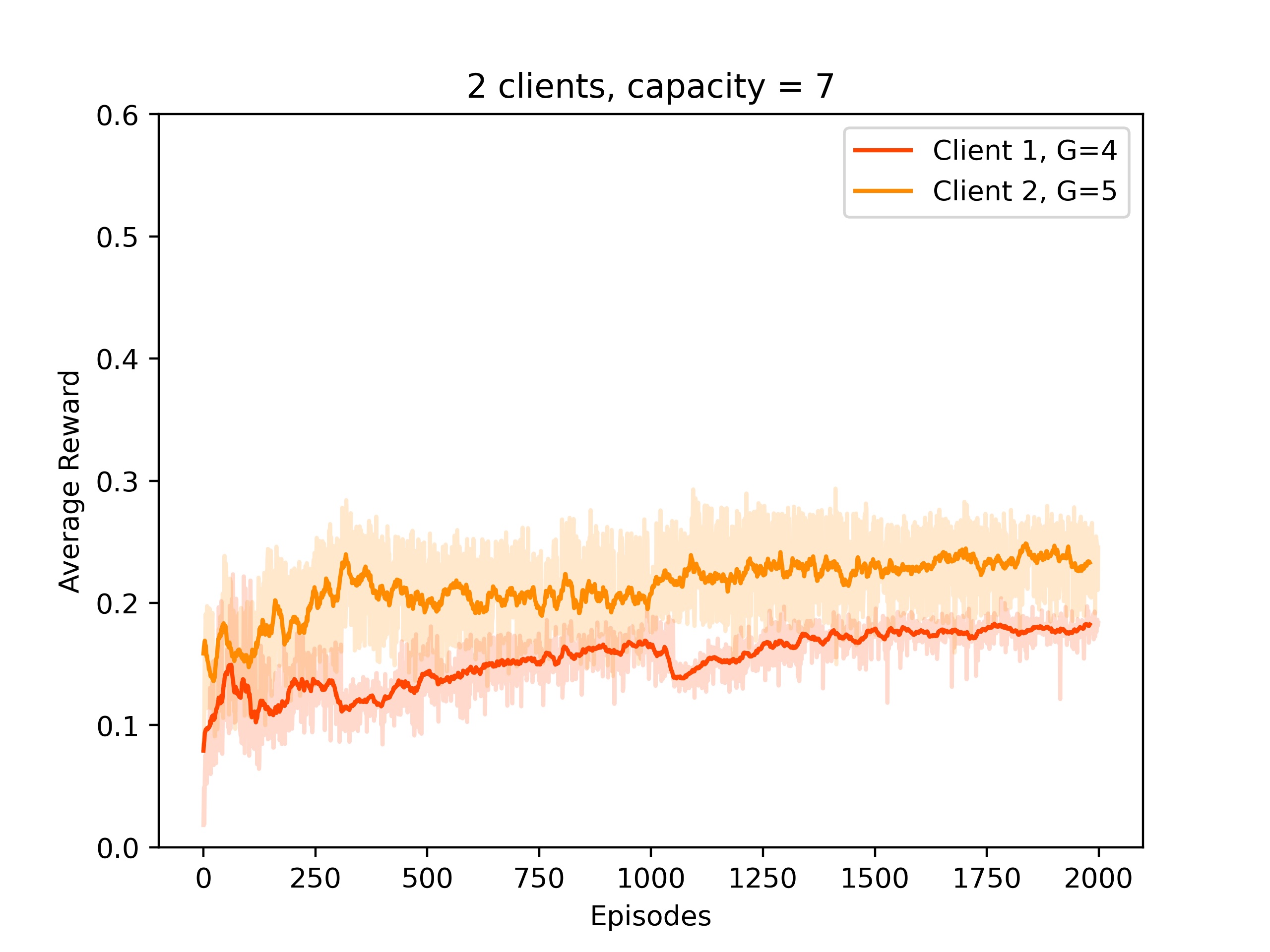} }}%
    \qquad
    \subfloat[\centering ]{{\includegraphics[width=7cm]{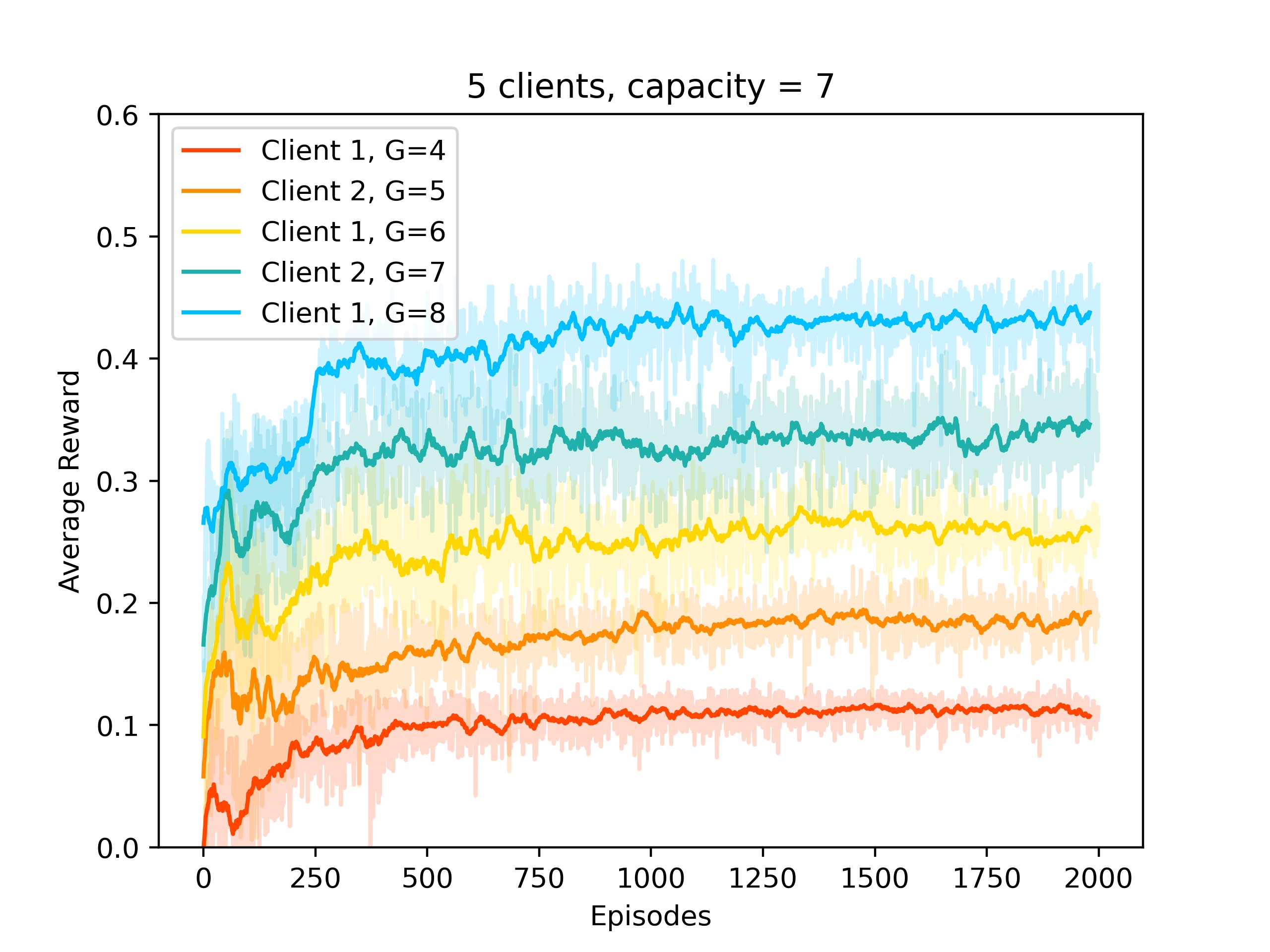} }}%
    \caption{Players' average reward trajectories versus the number of episodes ($\lambda=1.5$).}%
    \label{fig:lambda1.5}
\end{figure}

\section{Conclusions}\label{sec:conclusion}
In this work, we studied a subclass of stochastic games in which players have their own independent internal chains while they are coupled through their payoff functions. By establishing an equivalence between stationary NE in such games and NE points in a virtual continuous-action concave game, we developed scalable learning algorithms that converge to the set of $\epsilon$-NE policies in terms of the averaged Nikaido-Isoda gap function (for general rewards) and in terms of the Euclidean distance (for structured rewards). We also derived high probability or expected convergence rates that scale quadratically or logarithmically in terms of the game parameters in both cases. Beyond Markov potential games and linear-quadratic stochastic games, this work provides another interesting class of stochastic games that admit scalable learning algorithms under some assumptions.  

In general, there are strong computational lower bounds for developing scalable learning algorithms in $n$-player stochastic games. On the other hand, stochastic games provide natural paradigms for modeling competition under uncertainty. One approach to overcoming the computational barrier of learning NE in large-scale stochastic games is to rely on mean-field approximations or to study stochastic aggregative games \cite{zhang2021multi,uz2020reinforcement,meigs2019learning}. The main underlying assumption in mean-field games and aggregative games is that the individual actions of the players do not play a major role in the evolution of the state dynamics, but rather the mean/aggregate of their actions is the deriving force of the dynamics. Such an approach may simplify the learning task by allowing the players to focus on learning the mean-field trajectory of the actions/states rather than individual actions/states. However, the main difference between our work and those lines of work is that we target a more ambitious goal, i.e., learning in the space of all (exponentially many) action profiles. Interestingly, we showed that under certain assumptions, such as the independence of players' chains with structured rewards, one could avoid an exponential running time to learn the stationary NE policies, even for $n$ players, and without any mean-field approximation. In particular, our convergence rate results are not asymptotic, as is often the case for mean-field games. Therefore, studying other classes of stochastic games on the middle ground between mean-field/aggregative stochastic games and two-player stochastic games whose special structure allows scalable learning algorithms for computing NE policies is an interesting future research direction.

\bigskip
\noindent 
{\bf Acknowledgment:} I want to thank Tiancheng Qin for simulating Algorithm \ref{alg-main} and generating Figures \ref{fig:lambda0} and \ref{fig:lambda1.5}.

\bibliographystyle{IEEEtran}
\bibliography{thesisrefs}

\section*{Appendix I: Omitted Proofs and Auxiliary Lemmas}\label{sec:appx1}

\bigskip
{\bf \emph{Proof of Lemma \ref{lemm:original-to-shrunk}}:} Since each polytope $\mathcal{P}_i$ is a compact set and characterized by finitely many (continuous) linear constraints, for any $\epsilon_i:=\epsilon/\sqrt{|A_i||S_i|}$, there exist $\delta_i>0$, such that by replacing the constraint $\rho_i\ge 0$ in the description of $\mathcal{P}_i$ by $\rho_i\ge \delta_i\boldsymbol{1}$, the shrunk polytope $\mathcal{P}_i^{\delta_i}=\mathcal{P}_i\cap\{\rho_i\ge \delta_i \boldsymbol{1}\}$ has a maximum distance of at most $\epsilon_i$ from $\mathcal{P}_i$. In other words, for any point $\rho_i\in \mathcal{P}_i$, there exists $\hat{\rho}_i\in \mathcal{P}^{\delta_i}_i$ such that $\|\rho_i-\hat{\rho}_i\|\leq \epsilon_i$. Thus, by choosing $\delta=\min_i \delta_i$, for any $\rho\in \mathcal{P}$, there exists $\hat{\rho}\in \mathcal{P}^{\delta}=\prod_{i}\mathcal{P}_i^{\delta}$ such that $\|\rho_i-\hat{\rho}_i\|\leq \epsilon_i, \forall i$. 

Now, let $\rho^*$ be an $\epsilon$-NE for the virtual game with shrunk action polytope $\mathcal{P}^{\delta}$, and note that $\rho^*$ also belongs to $\mathcal{P}$. Consider any arbitrary $\rho\in \mathcal{P}$, and let $\hat{\rho}$ be the corresponding closest point to $\rho$ as described above.   Then, for any player $i$, we have
\begin{align}\nonumber
V_i(\rho_i,\rho^*_{-i})-V_i(\rho^*)&=\langle \rho_i-\rho^*_i, v_i(\rho^*)\rangle= \langle \rho_i-\hat{\rho}_i, v_i(\rho^*)\rangle +\langle \hat{\rho}_i-\rho^*_i, v_i(\rho^*)\rangle\cr 
&=\langle \rho_i-\hat{\rho}_i, v_i(\rho^*)\rangle+ \big(V_i(\hat{\rho}_i,\rho^*_{-i})-V_i(\rho^*)\big)\cr 
&\leq \|\rho_i-\hat{\rho}_i\|\|v_i(\rho^*)\|+\epsilon\leq \sqrt{|A_i||S_i|}\epsilon_i+\epsilon=2\epsilon, 
\end{align}
where the first inequality uses the fact that $\rho^*$ is an $\epsilon$-NE with respect to the shrunk polytope $\mathcal{P}^{\delta}$. Since $\rho\in \mathcal{P}$ was chosen arbitrarily, the above inequality shows that $\rho^*$ is a $2\epsilon$-NE with respect to the original polytope $\mathcal{P}$.  \hfill{$\blacksquare$}

\bigskip
{\bf \emph{Proof of Lemma \ref{lemm:almost-unbiased}}:} Given $\mathcal{F}^{k-1}$ and an arbitrary state-action $(s_i,a_i)\in S_i\times A_i$, let $\bar{t}$ be the first (random) time that the state $s_i$ is visited during the sampling interval $[\tau^k+d, \tau_i^k]$, i.e., $\bar{t}\in [\tau^k+d, \tau_i^k]$ is the first time for which $s_i^{\bar{t}}=s_i$. Since each state is sampled once, the expected value of the $(s_i,a_i)$-th coordinate of $R_i^k$ equals
\begin{align}\label{eq:indicator}
\mathbb{E}[R_i^k(s_i,a_i)|\mathcal{F}^{k-1}]=\mathbb{E}\big[\frac{r_i\big(s_i,a_i;s_{-i}^{\bar{t}},a_{-i}^{\bar{t}}\big)}{\pi_i^k(a_i|s_i)}\mathbb{I}_{\{a_i^{\bar{t}}=a_i\}}|\mathcal{F}^{k-1}\big]=\mathbb{E}\big[r_i\big(s_i,a_i;s_{-i}^{\bar{t}},a_{-i}^{\bar{t}}\big)|\mathcal{F}^{k-1}\big],
\end{align}
where $\mathbb{I}_{\{\cdot\}}$ is the indicator function, and the second equality holds because at time $\bar{t}$ the played action $a^{\bar{t}}_i$ equals $a_i$ with probability $\pi^k_i(a_i|s_i)$. Given $\mathcal{F}^{k-1}$, let $v\in \Delta(S_{-i})$ be the stationary distribution induced over states $s_{-i}$ by using policies $\pi^k_{-i}$. We can write
\begin{align}\nonumber
\mathbb{E}[R_i^k(s_i,&a_i)|\mathcal{F}^{k-1}]=\mathbb{E}\big[r_i\big(s_i,a_i;s_{-i}^{\bar{t}},a_{-i}^{\bar{t}}\big)|\mathcal{F}^{k-1}\big]\cr
&=\sum_{s_{-i},a_{-i}}\mathbb{P}(s_{-i}^{\bar{t}}=s_{-i}, a_{-i}^{\bar{t}}=a_{-i}|\mathcal{F}^{k-1})r_i(s_i,a_i;s_{-i},a_{-i})\cr 
&=\sum_{s_{-i},a_{-i}}\mathbb{P}(s_{-i}^{\bar{t}}=s_{-i}|\mathcal{F}^{k-1})\big(\prod_{j\neq i}\pi^k_j(a_j|s_j)\big) r_i(a,s)\cr
&=\sum_{s_{-i},a_{-i}}\big(\mathbb{P}(s_{-i}^{\bar{t}}=s_{-i}|\mathcal{F}^{k-1})-v(s_{-i})\big)\big(\prod_{j\neq i}\pi^k_j(a_j|s_j)\big) r_i(a,s)+\sum_{a_{-i},s_{-i}}v(s_{-i})\big(\prod_{j\neq i}\pi^k_j(a_j|s_j)\big) r_i(a,s)\cr
&=\sum_{s_{-i}}\big(\mathbb{P}(s_{-i}^{\bar{t}}=s_{-i}|\mathcal{F}^{k-1})-v(s_{-i})\big) \sum_{a_{-i}}\big(\prod_{j\neq i}\pi^k_j(a_j|s_j)\big) r_i(a,s)+\sum_{a_{-i},s_{-i}}\big(\prod_{j\neq i}\rho^k_j(s_j,a_j)\big) r_i(a,s)\cr 
&=\sum_{s_{-i}}\big(\mathbb{P}(s_{-i}^{\bar{t}}=s_{-i}|\mathcal{F}^{k-1})-v(s_{-i})\big) \sum_{a_{-i}}\big(\prod_{j\neq i}\pi^k_j(a_j|s_j)\big) r_i(a,s)+\nabla_{\rho_i(s_i,a_i)}V_i(\rho^k).
\end{align}
Given $\mathcal{F}^{k-1}$, let $v_{t}(\cdot)=\mathbb{P}(s_{-i}^t=\cdot|\mathcal{F}^{k-1})$ be the probability distribution of being at different states $s_{-i}$ at time $t$ should players $j\neq i$ follow policies $\pi^k_{-i}$. Since $r_i(a,s)\in [0, 1]$, using triangle inequality, we can write
\begin{align}\nonumber
\big|\mathbb{E}[R_i^k(s_i,a_i)|\mathcal{F}^{k-1}]-\nabla_{\rho_i(s_i,a_i)}V_i(\rho^k)\big|&\leq \sum_{s_{-i}}\big|\mathbb{P}(s_{-i}^{\bar{t}}=s_{-i}|\mathcal{F}^{k-1})-v(s_{-i})\big| \sum_{a_{-i}}\big(\prod_{j\neq i}\pi^k_j(a_j|s_j)\big)\cr
&=\sum_{s_{-i}}\big|\mathbb{P}(s_{-i}^{\bar{t}}=s_{-i}|\mathcal{F}^{k-1})-v(s_{-i})\big|\prod_{j\neq i}\big(\sum_{a_{j}}\pi^k_j(a_j|s_j)\big)\cr 
&=\sum_{s_{-i}}\big|\mathbb{P}(s_{-i}^{\bar{t}}=s_{-i}|\mathcal{F}^{k-1})-v(s_{-i})\big|\cr
&=\sum_{s_{-i}}\big|\sum_{t\ge \tau^k+d}\mathbb{P}(s_{-i}^t=s_{-i}|\mathcal{F}^{k-1})\mathbb{P}(\bar{t}=t|\mathcal{F}^{k-1})-v(s_{-i})\big|\cr
&=\sum_{s_{-i}}\big|\sum_{t\ge \tau^k+d}\big(\mathbb{P}(s_{-i}^t=s_{-i}|\mathcal{F}^{k-1})-v(s_{-i})\big)\mathbb{P}(\bar{t}=t|\mathcal{F}^{k-1})\big|\cr
&\leq\sum_{t\ge \tau^k+d}\sum_{s_{-i}}\big|\mathbb{P}(s_{-i}^t=s_{-i}|\mathcal{F}^{k-1})-v(s_{-i})\big|\mathbb{P}(\bar{t}=t|\mathcal{F}^{k-1})\cr
&= \mathbb{E}[\|v_{\bar{t}}-v\|_1|\mathcal{F}^{k-1}]\leq \mathbb{E}[e^{-\frac{(\bar{t}-\tau^k)}{\tau}}|\mathcal{F}^{k-1}],
\end{align}
where the third equality holds by the independency Assumption \ref{ass:indep},\footnote{Conditioned on $\mathcal{F}^{k-1}$, the event $\{\bar{t}=t|\mathcal{F}^{k-1}\}$ is determined by player $i$'s internal state transition matrix and a given policy $\pi^k_i$, which is independent from the event $\{s^t_{-i}=s_i|\mathcal{F}^{k-1}\}$, as the latter is determined by other players' state transition matrices and their given fixed policies $\pi^k_{-i}$.} and the last inequality holds by the ergodicity Assumption \ref{ass:ergodic}. Since $\bar{t}-\tau^k\ge d$ with probability 1, we have $\mathbb{E}[e^{-\frac{(\bar{t}-\tau^k)}{\tau}}|\mathcal{F}^{k-1}]\leq e^{-\frac{d}{\tau}}$. \hfill{$\blacksquare$}


\bigskip
{\bf \emph{Proof of Theorem \ref{thm:positive-converge-general}}:} Suppose by contradiction that $\rho^*$ is not an $\epsilon$-NE for the virtual game, and let $v_i(\rho)=\nabla_{\rho_i}V_i(\rho)$. Since the virtual game is a concave game, using the NE characterization for concave games, there exists a player $i$ and deviation strategy $\hat{\rho}_i\in \mathcal{P}^{\delta_i}_i$ such that $\langle v_i(\rho^*), \hat{\rho}_i-\rho^*_i\rangle>\epsilon$. Thus, if we let $\alpha=\frac{\epsilon}{3\sqrt{|A_i||S_i|}}$, for any $v'_i$ and $\rho'_i\in \mathcal{P}^{\delta_i}_i$ such that $\|v'_i-v_i(\rho^*)\|<\alpha$ and $\|\rho'_i-\rho^*_i\|<\alpha$, we have $\langle v'_i, \hat{\rho}_i-\rho'_i\rangle>\frac{\epsilon}{3}$. To show that, we note that each coordinate of $v_i(\rho^*)$ lies in $[-1, 1]$ because 
\begin{align}\nonumber
\big|v_i(\rho^*)_{(s_i,a_i)}\big|&=\big|\sum_{s_{-i},a_{-i}}\prod_{j\neq i}\rho^*_j(s_j,a_j)r_i(s,a)\big|\leq \sum_{s_{-i},a_{-i}}\prod_{j\neq i}\rho^*_j(s_j,a_j)=1.
\end{align}
Thus $\|v_i(\rho^*)\|\leq \sqrt{|A_i||S_i|}$, and we can write  
\begin{align}\nonumber
\langle v'_i, \hat{\rho}_i-\rho'_i\rangle&=\langle v'_i-v_i(\rho^*), \hat{\rho}_i-\rho'_i\rangle+\langle v_i(\rho^*), \hat{\rho}_i-\rho^*_i\rangle+\langle v_i(\rho^*), \rho^*_i-\rho'_i\rangle\cr 
&>\langle v'_i-v_i(\rho^*), \hat{\rho}_i-\rho'_i\rangle+\epsilon+\langle v_i(\rho^*), \rho^*_i-\rho'_i\rangle\cr 
&>-\alpha \|\hat{\rho}_i-\rho'_i\|+\epsilon-\|v_i(\rho^*)\|\alpha\cr
&\ge \epsilon-(\sqrt{2}+\sqrt{|A_i||S_i|})\alpha\ge \frac{\epsilon}{3},
\end{align} 
where the second inequality uses the Cauchy-Schwartz inequality, and the last inequality holds by the choice of parameter $\alpha$.  

Let us define the (coordinatewise) martingale difference sequence $X^{\ell}_i=\eta^{\ell}_i(R^{\ell}_i-\mathbb{E}[R^{\ell}_i|\mathcal{F}^{\ell-1}]), \ell=1,2,\ldots$, and the corresponding zero-mean martingale $S^k_i=\sum_{\ell=1}^kX^{\ell}_i$. Consider the nondecreasing sequence of positive numbers $w^k_i=\sum_{\ell=1}^{k}\eta^{\ell}_i$. Then, we can write
\begin{align}\nonumber
\sum_{\ell=1}^{\infty}\frac{\mathbb{E}[\|X^{\ell}_i\|^2|\mathcal{F}^{\ell-1}]}{(w^{\ell}_i)^2}&=\sum_{\ell=1}^{\infty}\big(\frac{\eta^{\ell}_i}{w^{\ell}_i}\big)^2\mathbb{E}\big[\big\|R^{\ell}_i-\mathbb{E}[R^{\ell}_i|\mathcal{F}^{\ell-1}]\big\|^2\big|\mathcal{F}^{\ell-1}\big]\cr 
&=\sum_{\ell=1}^{\infty}\big(\frac{\eta^{\ell}_i}{w^{\ell}_i}\big)^2\big(\mathbb{E}[\|R^{\ell}_i\|^2|\mathcal{F}^{\ell-1}]-\|\mathbb{E}[R^{\ell}_i|\mathcal{F}^{{\ell}-1}]\|^2\big)\cr
&\leq \sum_{\ell=1}^{\infty}\big(\frac{\eta^{\ell}_i}{w^{\ell}_i}\big)^2\big(\mathbb{E}[\|R^{\ell}_i\|^2|\mathcal{F}^{\ell-1}]\big)\cr 
&=\sum_{\ell=1}^{\infty}\big(\frac{\eta^{\ell}_i}{w^{\ell}_i}\big)^2\sum_{s_i,a_i}\mathbb{E}\big[\frac{r^2_i(s_i,a_i;s^{\bar{t}}_{-i},a^{\bar{t}}_{-i})}{\pi_i^{\ell}(a_i|s_i)}|\mathcal{F}^{\ell-1}\big]\cr 
&\leq \sum_{\ell=1}^{\infty}\big(\frac{\eta^{\ell}_i}{w^{\ell}_i}\big)^2\frac{|A_i||S_i|}{\delta_i}, 
\end{align}  
where the last equality is obtained using an argument similar to that used to derive \eqref{eq:indicator}, and the final inequality holds because $r^2_i(\cdot)\leq 1$ and $\pi^{\ell}_i(a_i|s_i)\ge \rho_i^{\ell}(s_i,a_i)\ge \delta_i, \forall i,s_i,a_i$. Since by the step-size assumption $\sum_{\ell=1}^{\infty}\big(\frac{\eta^{\ell}_i}{w^{\ell}_i}\big)^2<\infty$ and $\lim_{k\to \infty}w_i^k=\infty$, using the martingale convergence theorem, almost surely we have 
\begin{align}\label{eq:martingle-conv}
\lim_{k\to \infty}\frac{S^k_i}{w^k_i}=\lim_{k\to \infty} \sum_{\ell=1}^k\frac{\eta^{\ell}_i}{w^k_i}(R^{\ell}_i-\mathbb{E}[R^{\ell}_i|\mathcal{F}^{\ell-1}])=0.
\end{align}
Let $\Omega$ be the event that $\rho^k$ converges to $\rho^*$. Conditioned on $\Omega$, almost surely we get 
\begin{align}\label{eq:Y}
\|\frac{Y^{k+1}_i}{w_i^k}-v_i(\rho^*)\|&=\|\sum_{\ell=1}^k\frac{\eta_i^{\ell}}{w_i^k}R^{\ell}_i-v_i(\rho^*)\|\cr 
&=\|\sum_{\ell=1}^k\frac{\eta_i^{\ell}}{w_i^k}\big(\mathbb{E}[R^{\ell}_i|\mathcal{F}^{\ell-1}]-v_i(\rho^{\ell})+v_i(\rho^{\ell})-v_i(\rho^*)\big)+\frac{S^k_i}{w_i^k}\|\cr 
&\leq \sum_{\ell=1}^k\frac{\eta_i^{\ell}}{w_i^k}\|\mathbb{E}[R^{\ell}_i|\mathcal{F}^{\ell-1}]-v_i(\rho^{\ell})\|+\sum_{\ell=1}^k\frac{\eta_i^{\ell}}{w_i^k}\|v_i(\rho^{\ell})-v_i(\rho^*)\|+\|\frac{S^k_i}{w_i^k}\| \cr 
& \leq \sum_{\ell=1}^k\frac{\eta_i^{\ell}}{w_i^k}e^{-\frac{d}{\tau}}\|\boldsymbol{1}\|+\sum_{\ell=1}^k\frac{\eta_i^{\ell}}{w_i^k}\|v_i(\rho^{\ell})-v_i(\rho^*)\|+\|\frac{S^k_i}{w_i^k}\|\cr
&\stackrel{k\to \infty}{=} e^{-\frac{d}{\tau}}\sqrt{|A_i||S_i|},
\end{align}
where the last equality holds by Lemma \ref{lemm:almost-unbiased}, by relation \eqref{eq:martingle-conv}, and because $\lim_{k\to \infty}v_i(\rho^k)=v_i(\rho^*)$ due to continuity of $v_i(\cdot)$. Using \eqref{eq:Y} and since $d\ge\tau\ln (\frac{3\max_i|A_i||S_i|}{\epsilon})$, we obtain $\mathbb{P}\big(\limsup_k\|\frac{Y^{k+1}_i}{w_i^k}-v_i(\rho^*)\|<\alpha\big| \Omega\big)=1$, which, together with $\mathbb{P}\big(\lim_k\|\rho^{k+1}_i-\rho^*_i\|<\alpha\big| \Omega\big)=1$, implies that almost surely
\begin{align}\label{eq:epsilon/2}
\limsup_k\langle \frac{Y^{k+1}_i}{w_i^k}, \hat{\rho}_i-\rho^{k+1}_i\rangle>\frac{\epsilon}{3}.
\end{align}
Moreover, because $\rho^{k+1}_i=\argmax_{\rho_i\in \mathcal{P}^{\delta_i}_i}\{\langle \rho_i, Y^{k+1}_i\rangle -h_i(\rho_i)\}$, we have $\langle \rho^{k+1}_i, Y^{k+1}_i\rangle -h_i(\rho^{k+1}_i)\ge \langle \hat{\rho}_i, Y^{k+1}_i\rangle -h_i(\hat{\rho}_i)$. Therefore, for all sufficiently large $k$, almost surely we have
\begin{align}\nonumber
h_i(\hat{\rho}_i)-h_i(\rho^{k+1}_i)\ge \langle \hat{\rho}_i-\rho^{k+1}_i, Y^{k+1}_i\rangle=w_i^{k} \langle \hat{\rho}_i-\rho^{k+1}_i, \frac{Y^{k+1}_i}{w_i^{k}}\rangle=w_i^{k} \frac{\epsilon}{3}.
\end{align}
Since $\mathbb{P}(\Omega)>0$ by the assumption and within $\Omega$ we have $h_i(\rho^{k+1}_i)\to h_i(\rho^*_i)$, with positive probability we must have $h_i(\hat{\rho}_i)\ge h_i(\rho^*_i)+w_i^{k} \frac{\epsilon}{3}=\infty$ as $k\to \infty$. This contradiction shows that $\rho^*$ must be an $\epsilon$-NE. \hfill{$\blacksquare$}

\bigskip
\subsection*{{\bf {\large Auxiliary Lemmas}}} 

\bigskip
\begin{lemma}\label{lemm:indep}
Let Assumption \ref{ass:indep} hold. Then, $\mathbb{P}(s^t=s)=\prod_{j=1}^n\mathbb{P}(s_j^t=s_j), \forall t,s$.  
\end{lemma}
\begin{proof}
We use an induction on $t$. For the initial step $t=0$, the statement trivially holds. We can write 
\begin{align}\nonumber
&\mathbb{P}(s^t=s)=\sum_{s',a'}\mathbb{P}(s^t=s|s^{t-1}=s',a^{t-1}=a')\mathbb{P}(a^{t-1}=a'|s^{t-1}=s')\mathbb{P}(s^{t-1}=s')\cr 
&=\sum_{s',a'}\prod_j \Big(P_j(s_j^t=s_j|s_j^{t-1}=s_j',a_j^{t-1}=a_j')\pi_j(a_j^{t-1}=a_j'|s_j^{t-1}=s_j')\Big)\mathbb{P}(s^{t-1}=s')\cr 
&=\sum_{s',a'}\prod_j \Big(P_j(s_j^t=s_j|s_j^{t-1}=s_j',a_j^{t-1}=a_j')\pi_j(a_j^{t-1}=a_j'|s_j^{t-1}=s_j')\mathbb{P}(s_j^{t-1}=s_j')\Big)\cr 
&=\sum_{s',a'}\prod_j \Big(P_j(s_j^t=s_j|s_j^{t-1}=s_j',a_j^{t-1}=a_j')\mathbb{P}(s_j^{t-1}=s_j',a_j^{t-1}=a_j')\Big)\cr 
&=\prod_j \Big(\sum_{s_j',a_j'}P_j(s_j^t=s_j|s_j^{t-1}=s_j',a_j^{t-1}=a_j')\mathbb{P}(s_j^{t-1}=s_j',a_j^{t-1}=a_j')\Big)=\prod_j \mathbb{P}(s_j^t=s_j), 
\end{align}
where the second and third equalities use Assumption \ref{ass:indep} and the hypothesis.  
\end{proof}

\begin{lemma}[\cite{mertikopoulos2019learning}, Proposition 4.3]\label{lemm:fenchel}
Let $h(x)$ be a $K$-strongly convex function over a convex compact set $\mathcal{X}$, and $h^*$ be its convex conjugate defined by $h^*(y)=\max_{x\in \mathcal{X}}\{\langle x, y\rangle-h(x)\}$. Let $F(p,y)=h(p)+h^*(y)-\langle y, p\rangle$ be the Fenchel coupling function. Then, $F(p,y)\ge 0, \forall p\in \mathcal{X}, \forall y$, and we have
\begin{align*}
F(p,y')\leq F(p,y)+\langle y'-y, \Pi(y)-p\rangle+\frac{1}{2K}\|y'-y\|^2,\  \forall p\in \mathcal{X}, \forall y,y',
\end{align*}
where $\Pi(y)=\argmax_{x\in \mathcal{X}}\{\langle x, y\rangle-h(x)\}$.
\end{lemma}

\bigskip
\begin{lemma}[\cite{wang2017primal}, Lemma 4]\label{lemm-KL}
Let $\Delta$ be an $m$-dimensional probability simplex and $X$ be a compact and convex subset of $\Delta$. Moreover, let $y=(y_1,\ldots,y_m)$ be an arbitrary nonpositive vector and assume $x^k\in X$. Consider the following two-step MD update rule:
\begin{align}\nonumber
&x^{k+\frac{1}{2}}=\argmax_{x\in \Delta}\{\langle y, x \rangle-D_{KL}(x,x^{k})\},\cr 
&x^{k+1}=\argmin_{x\in X} D_{KL}(x,x^{k+\frac{1}{2}}).
\end{align}
Then, for any $z\in X$, we have $D_{KL}(z,x^{k+1})-D_{KL}(z,x^k)\leq \langle y, x^k-z \rangle+\frac{1}{2}\langle y^2, x^k \rangle$, where $y^2=(y^2_1,\ldots,y_m^2)$.
\end{lemma}

\end{document}